\documentclass{article}

\usepackage[preprint]{neurips_2026}

\usepackage[utf8]{inputenc} 
\usepackage[T1]{fontenc}    
\usepackage{hyperref}       
\usepackage{url}            
\usepackage{booktabs}       
\usepackage{amsfonts}       
\usepackage{nicefrac}       
\usepackage{microtype}      
\usepackage{xcolor}         

\title{Reinforcement Learning for Exponential Utility: Algorithms and Convergence in Discounted MDPs}

\author{%
  Gugan Thoppe \\
  Computer Science and Automation (CSA)\\
  Indian Institute of Science (IISc)\\
  Bengaluru 560012, India \\
  \texttt{gthoppe@iisc.ac.in} \\
  \And
  Prashanth L. A.\\
  Computer Science and Engineering, \\
  Indian Institute of Technology Madras \\
  Chennai 600036, India\\
  \texttt{prashla@cse.iitm.ac.in} \\
  \AND
  Ankur Naskar \\
  Computer Science and Automation (CSA)\\
  Indian Institute of Science (IISc) \\
  Bengaluru 560012, India \\
  \texttt{ankurnaskar@iisc.ac.in} \\
  \And
  Sanjay P. Bhat\\
  TCS Research, \\
  IIIT-H Research Park, \\ 
  Hyderabad 500032, India\\
  \texttt{sanjay.bhat@tcs.com} \\
}

\usepackage{macros}

\begin{document}

\maketitle

\begin{abstract}
    Risk-sensitive reinforcement learning with exponential utility in discounted Markov decision processes remains unexplored from a model-free, value-based perspective. We address this gap in the risk-aversion setting. Building on a Bellman-type equation for exponential utility from \citep{porteus1975optimality}, we derive two Q-value-style extensions and show that the associated operators are contractions in the $\|\cdot\|_\infty$ and Thompson metrics, respectively. We characterize the fixed points of these operators and show that the induced greedy stationary policy minimizes an upper bound on the exponential-utility objective among stationary policies. Using these structural results, we develop two model-free algorithms: a two-timescale Q-learning variant and a one-timescale algorithm based on a sublinear power-law operator. We prove almost-sure convergence under Markovian sampling for both algorithms. While the two-timescale method is analyzed using timescale separation, the one-timescale variant needs a different approach based on local Lipschitzness, monotonicity, homogeneity, and Dini derivatives, as its operator is not globally contractive in a standard metric. We further derive finite-time convergence rates for the two-timescale method under i.i.d. sampling and for the one-timescale method in the scalar case. Our work opens up new possibilities for model-free learning of optimal policies for exponential-utility objectives, beyond the class of stationary policies.
\end{abstract}
 

\section{Introduction}
    The traditional objective in reinforcement learning (RL) \citep{BertsekasT96,puterman1994markov,sutton2018reinforcement} is to find a policy that maximizes the expected cumulative reward. However, two policies with the same expected return may have very different return distributions, and hence different levels of risk. This is particularly important in applications such as finance, healthcare, and robotics. Risk-sensitive RL \citep{MAL-091}, which has garnered increased research attention over the last decade, aims to account for such differences. Broadly, risk-sensitive RL accounts for risk either through constraints or by incorporating risk directly into the objective. In the constrained setting, one maximizes the expected return subject to a constraint on a chosen risk measure. In the unconstrained setting, one optimizes an objective that directly accounts for risk. Risk measures studied in RL include variance \citep{Mihatsch02RS,tamar2012policy} \mbox{\citep{prashanth2016variance}}, quantile \citep{jiang2017risk}, conditional value at risk (CVaR) \citep{prashanth2014cvar,thoppe2024risk}, distortion risk \citep{markowitz2023risk}, and coherent risk \citep{tamar2016sequential}. In this work, we focus on exponential utility as an unconstrained risk-sensitive objective  \citep{bauerle2014}.

    Exponential-utility optimization has been studied extensively in average-reward MDPs \citep{borkar2010learning,moharrami2022policy}. In discounted MDPs, however, the Bellman recursion is time-inhomogeneous: the optimal value at risk sensitivity $\lambda$ depends on that at $\gamma\lambda$, which in turn depends on that at $\gamma^2\lambda$ and so on \citep{jaquette1976utility,chung1987discounted}. Thus, computing the optimal value at $\lambda$ requires determining the optimal values at infinitely many risk-sensitivity parameters. This contrasts with the risk-neutral setting, where the optimal value is characterized by a single fixed-point equation. An optimal policy for exponential utility, moreover, need not be stationary. While it can be chosen to be ultimately stationary \citep{jaquette1976utility}, the stationarity horizon is generally unknown. Even if this horizon can be approximated, computing an optimal stationary tail remains challenging, as the optimal value over stationary policies generally satisfies only a time-inhomogeneous Bellman inequality. Existing model-based methods \citep{kumar2019ultimately,hau2023entropic} address this difficulty by using an optimal risk-neutral stationary policy as the tail and its value to initialize a finite backward Bellman recursion for computing the preceding values and decision rules. However, this approach ignores the residual risk sensitivity at the chosen horizon. Our work asks whether optimal risk-sensitive stationary policies can be learned using model-free RL.  

    A different formulation of discounted risk-sensitive control based on exponential utility was studied by \citet{porteus1975optimality}, who imposed an interstage consistency requirement. This requirement keeps the risk-sensitivity parameter fixed across the original Bellman recursion, leading to two time-homogeneous Bellman fixed-point equations, one for each stationary policy and another obtained by optimizing over actions. Porteus further showed that a stationary policy greedy with respect to the latter fixed point is optimal for the criterion defined by his Bellman recursion. However, how this policy performs under the direct exponential utility of the discounted return remains unclear. In this work, we study the Porteus formulation for $\lambda>0$ and ask the following questions:

    \begin{enumerate}[leftmargin=*]
        \item \textit{Can we develop model-free RL algorithms to solve the Porteus Bellman optimality equation?}
        
        \item \textit{What relationship can we establish between the Porteus formulation and the direct exponential-utility objective?}
    \end{enumerate}

    We address these questions through the following \textbf{key contributions}.
    \begin{itemize}[leftmargin=*]
        
        \item \textbf{Bellman-type operators:}   
        Building on  \citet{porteus1975optimality}'s state-value Bellman  formulation, we study two Q-value extensions, covering both policy evaluation and optimality. We show that the associated operators are contractions in the $\|\cdot\|_\infty$ and sup-log/Thompson metrics, respectively. These results provide the foundation for our model-free RL algorithms.
        
        \item \textbf{Upper bound for direct exponential utility:}
        We show that, for any stationary policy, the fixed point of our policy-specific Q-value operator provides an upper bound on its direct exponential-utility Q-values. We further show that the greedy policy obtained from the optimal Q-value fixed point minimizes these upper bounds among all stationary policies.
    
        \item \textbf{Q-learning--style algorithms:} We develop the first tabular, model-free Q-learning--style algorithms for learning risk-sensitive stationary policies under the Porteus formulation in discounted MDPs. We propose two algorithms, one based on two-timescale stochastic approximation and the other on one-timescale stochastic approximation.
    
        \item \textbf{Two-timescale analysis:}
        For the proposed two-timescale algorithm, we establish almost-sure convergence under Markovian sampling and finite-time convergence rates under i.i.d. sampling, using timescale-separation arguments. These results extend classical Q-learning-style convergence guarantees to the exponential-utility setting.
    
        \item \textbf{One-timescale analysis:}
        For our one-timescale algorithm, whose dynamics are governed by a sublinear power-law operator, we prove almost-sure convergence under Markovian sampling using a novel analysis based on local Lipschitzness, monotonicity, homogeneity, and Dini derivatives. Under i.i.d. sampling, we also provide finite-time analysis of an analogous scalar recursion, that highlights the challenges due to the lack of a global contraction in standard metrics.
    \end{itemize}

    \paragraph{Related work.} Exponential-utility optimization has been studied extensively in average-reward MDPs. \citet{borkar2002q,borkar2001sensitivity} developed tabular, model-free Q-learning and actor--critic algorithms, while \citet{basu2008learning} studied the linear function approximation case. More recently, \citet{moharrami2022policy} developed a trajectory-based policy-gradient algorithm, and \citet{pmlr-v211-murthy23a} established convergence guarantees for model-based modified policy iteration. See also \citet{borkar2010learning} for a survey. Of particular relevance to our work is the Q-learning algorithm of \citet{borkar2002q}, which uses exponential-coordinate updates similar to our one-timescale method.

    For exponential utility with discounted returns, \citet{jaquette1976utility} showed that the Bellman recursion is time-inhomogeneous, as the risk parameter changes over time. \citet{chung1987discounted} showed that Markov policies still suffice for optimality, but need not be stationary. \citet{jaquette1976utility} further established the existence of optimal ultimately stationary policies, although the stationarity horizon depends on typically unknown problem-specific quantities. More recently, \citet{bauerle2026policy} revisited these results with simpler proofs and studied policy stability. Model-based approaches for approximating optimal policies were developed by \citet{kumar2019ultimately,hau2023entropic}. For a chosen finite horizon, they use an optimal risk-neutral stationary policy as the tail and backward Bellman recursions to compute the preceding time-dependent decision rules. To approximate optimal stationary policies, \citet{porteus1975optimality} introduced a time-homogeneous Bellman equation based on a fixed-point version of the exponential-utility Bellman recursion. \citet{bauerle2022markov} extended similar results to more general risk measures. For an overview, see \citet{bauerle2024markov}.
    
    Model-free algorithms for exponential utility have also been developed for finite-horizon MDPs. \citet{fei2020risk} proposed risk-sensitive value-iteration and Q-learning algorithms with regret guarantees, while \citet{fei2021exponential} obtained improved guarantees using an exponential Bellman equation. Extensions to function approximation were developed by \citet{fei2021risk}, and \citet{liang2024bridging} established regret bounds for model-free and model-based distributional RL algorithms with entropic-risk objectives. These works focus on episodic problems.

    More closely related to our work, \citet{mortensen2026recursive} established finite-sample guarantees for model-based Q-value iteration under recursive entropic risk in discounted MDPs. Their method estimates transition probabilities using a generative model, whereas ours learns the Bellman fixed point directly from sampled transitions.

    Model-free policy-gradient and actor--critic methods for exponential utility have also been studied \citep{noorani2025expCostPG,granados2026risk}, including applications to large language model alignment \citep{jiang2025risk}. Unlike these policy-based methods, our Q-learning--style algorithms for the Porteus formulation have almost-sure convergence and finite-time guarantees. We also connect the Porteus fixed points to the direct exponential utility of discounted returns.

\section{Discounted reward MDP with exponential utility}

We consider a discounted MDP $\cM = (\cS, \cA, \cP, r, \gamma).$ Here, $\cS$ and $\cA$ are finite state and action spaces with cardinalities $S=|\cS|$ and $A=|\cA|$, respectively. Further, $\cP: \cS \times \cA \to \Delta(\cS)$ is the transition kernel, where $\Delta(\cU)$ denotes the set of distributions over $\cU;$ in particular, $\cP(\sp|s, a)$ specifies the probability that the state of the  MDP changes from $s$ to $\sp$ under action $a.$ Finally, $r: \cS \times \cA \to \bR$ is the  reward function, and $\gamma \in (0, 1)$ the discount factor. Our goal is to develop model-free Q-learning--style algorithms for stationary policies under exponential utility.

Let $\Pi$ denote the set of all policies, including randomized and history-dependent ones. For each $\pi\in\Pi$ and risk-sensitivity parameter $\lambda>0$, define its direct exponential utility $X_\pi(\lambda)\in\bRp^{SA}$ by $X_\pi(s,a;\lambda)
=
\mathbb{E}_{s,a}^{\pi}\left[
\exp\left(
-\lambda\sum_{t=0}^{\infty}\gamma^t r(s_t,a_t)
\right)
\right],$ where $\bRp^{SA} := (0, \infty)^{SA}$ is the positive orthant in $\bR^{SA}$ and $\bE_{s,a}^{\pi}$ denotes expectation under the process
initialized at $(s_0,a_0)=(s,a)$, with policy $\pi$ followed
from time $1$ onward. The direct exponential-utility objective is to minimize $X_\pi(\lambda)$ over $\pi\in\Pi$. Following the arguments of  \citep{jaquette1976utility} (see \citep[(8)]{chung1987discounted}), the optimal value $X^*(s,a;\lambda)=\inf_{\pi\in\Pi}X_\pi(s,a;\lambda)$ satisfies the time-inhomogeneous Bellman recursion
\[
    X^*(s,a;\lambda) = e^{-\lambda r(s,a)} \sum_{\sp\in\cS}\cP(\sp\mid s,a) \min_{\ap\in\cA}X^*(\sp,\ap;\gamma\lambda).
\]
This recursion is time-inhomogeneous because the effective risk parameter changes over time. This implies the optimal decision rule may change even though the MDP is stationary. Indeed, \citet{chung1987discounted} showed that Markov policies suffice for optimality, but need not be stationary. Nevertheless, \citet{jaquette1976utility} established the existence of optimal ultimately stationary policies, i.e., policies which become stationary after finitely many steps. However, this stationarity horizon depends on typically unknown problem-specific quantities.

To approximate such policies, \citet{kumar2019ultimately,hau2023entropic} developed model-based methods that use a risk-neutral stationary policy beyond a chosen horizon $k$ and compute preceding decision rules through backward Bellman recursion. However, the risk-neutral tail ignores the residual risk sensitivity $\lambda\gamma^k$. Motivated by this limitation, we seek a risk-sensitive policy that is stationary from the outset, corresponding to a stationarity horizon of $k=0$.

Denote the set of all stationary Markovian policies (including randomized ones) for the MDP $\cM$ by $\Stat$. 
Optimizing direct exponential utility over $\Stat$ is challenging, as the time-inhomogeneous Bellman recursion does not directly characterize the optimal stationary policy. We therefore consider an alternative formulation originating from \citet{porteus1975optimality}, which we present in Q-value form. To motivate this formulation, define $Q^D_*(s,a;\lambda) = -\frac{1}{\lambda}\ln X^*(s,a;\lambda)$. The preceding Bellman recursion concerning $X^*$ then becomes
\[
    Q^D_*(s,a;\lambda) = r(s,a)-\frac{1}{\lambda}\ln\left[\sum_{\sp\in\cS}\cP(\sp\mid s,a) \exp\left(-\lambda\gamma\max_{\ap\in\cA} Q^D_*(\sp,\ap;\gamma\lambda)\right) \right].
\]
To avoid this time-inhomogeneity, Porteus's formulation leads to the fixed-point equation.
\[
    Q(s,a)
    =r(s,a)-\frac{1}{\lambda}\ln\left[
    \sum_{\sp\in\cS}\cP(\sp\mid s,a)
    \exp\left(-\lambda\gamma\max_{\ap\in\cA}Q(\sp,\ap)\right)
    \right].
\]
We next study this recursive formulation, establish the existence of an optimal stationary policy for this formulation, and examine its relationship with direct exponential utility.

To analyze the Porteus formulation, we work in exponential coordinates $x(s,a)=e^{-\lambda Q(s,a)}$. Then, the preceding fixed-point equation becomes $x=F(x)$, where $F:\bR_{++}^{SA}\to\bR_{++}^{SA}$ is defined by
\begin{align}
    & F(x)(s, a) = \exp\left(-\lambda r(s, a)\right) 
    \sum_{\sp \in \cS} \cP(\sp|s, a)  \left[\min_{\ap} x(\sp, \ap)\right]^\gamma.  \label{e:F.defn}
\end{align}
For any stationary policy $\pi \in \Stat,$ we also define the operator $F_{\pi}:\bRp^{SA} \to \bRp^{SA}$ by 
\begin{align}
    & F_{\pi}(x)(s, a) = \exp\left(- \lambda r(s, a)\right) 
    \sum_{(\sp,\ap) \in \cS\times \cA} \cP(\sp|s, a) \pi(\ap|\sp) \left[ x(\sp, \ap)\right]^\gamma.  \label{e:Fpi.defn}
\end{align} 
Consider the sup-log or Thompson metric $\dx{\cdot}{\cdot}:\bRp^{SA}\times \bRp^{SA}\rightarrow [0,\infty)$ on $\bRp^{SA}$ defined by
\begin{equation} \label{e:dx.norm.defn} \dx{x_1}{x_2} := \|\ln x_1 - \ln x_2\|_\infty, \qquad  x_1,x_2\in \bRp^{SA}. \end{equation} %
It is not difficult to see that $\bRp^{SA}$ is a complete metric space under the sup-log metric $d$. Our first result shows that, under this geometry, the operators $F$ and $F_{\pi}$ satisfy the structural properties that drive the rest of our analysis. For any $x_1, x_2 \in \bRp^{SA},$ let $x_1\preceq x_2$ (equivalently, $x_2\succeq x_1$) denote that $x_1(s,a) \leq x_2(s,a)$ for every $(s,a)\in\cS\times\cA$.

\begin{proposition}
    [\textbf{\textit{Structural properties of $F$ and $F_{\pi}$}}]
\label{prop:TFcontract1}
%
%
The following statements are true. 
\begin{enumerate}[leftmargin=*]
    \item $F$ is a $\gamma$-contraction in $\dx{\cdot}{\cdot}$, that is,  $\forall x_1, x_2 \in \bRp^{SA}$ we have
    $
        \dx{F(x_1)}{F(x_2)} \leq \gamma \dx{x_1}{x_2}.
    $
    
    \item For each $\pi\in\Stat$, $F_{\pi}$ is a $\gamma$-contraction in $\dx{\cdot}{\cdot}$.
    
    \item Let $\pi\in \Stat$. Then $F$ and $F_{\pi}$ are monotone, that is, for every $x_1,x_2\in\bRp^{SA}$ such that $x_1\preceq x_{2}$, we have $F(x_1)\preceq F(x_2)$ and $F_{\pi}(x_1)\preceq F_{\pi}(x_2)$. 
    \item For each $x\in\bRp^{SA}$, we have $F(x)=\min_{\pi\in\Stat}F_{\pi}(x)$, where the minimum is taken element-wise. 
    \item $F$ is homogeneous, that is, for each $c>0$ and $x\in\bRp^{SA}$, $F(cx)=c^{\gamma}F(x)$ .
\end{enumerate}
\end{proposition}

Proposition~\ref{prop:TFcontract1} identifies the geometry under which $F$ and $F_\pi$, $\pi\in\Stat$, are contractive. Since $\bRp^{SA}$ is complete under $d$, Banach's fixed-point theorem implies that $F$ and each $F_\pi$ admit unique fixed points, denoted by $\xstr$ and $\bar x_\pi$, respectively. Let $\pi^*\in\Stat$ be a stationary deterministic policy that selects an action in $\argmin_{a\in\cA}\xstr(s,a)$ at each state $s\in\cS$. Our next main result establishes that $\pi^*$ is optimal for the recursive formulation, in the sense that $\bar x_{\pi^*}=\xstr\preceq\bar x_\pi$ for every $\pi\in\Stat$. It also shows that $\bar x_\pi$ provides an upper bound on the direct exponential utility $X_\pi(\lambda)$.

\begin{theorem}
\label{thm:motiv}
Let $x^*$ and $\pi^*$ be as defined above. Then, for every
$\pi\in\Stat$, we have
\begin{equation}
\label{eq:main}
X_\pi(\lambda)\preceq\bar{x}_\pi
\qquad\text{and}\qquad
x^*=\bar{x}_{\pi^*}\preceq\bar{x}_\pi.
\end{equation}
\end{theorem}

\begin{remark}
The Porteus-optimal stationary policy need not minimize direct exponential utility, even among stationary policies (Appendix~\ref{sec:counterexample-bar-x-X}
). However, for finite MDPs with nonnegative rewards and uniformly absorbing dynamics, the optimal values of the two criteria differ by at most $O(1-\gamma)$ as $\gamma\uparrow1$ (Appendix~\ref{sec:absorbing-agreement}).
\end{remark}

Theorem~\ref{thm:motiv} establishes the optimality of $x^*$ for the Porteus formulation and shows that $\bar x_\pi$ upper-bounds the direct exponential utility $X_\pi(\lambda)$. We next develop two model-free algorithms for approximating $x^*$, one of which uses a Q-value operator associated with \citet{porteus1975optimality}. For the subsequent analysis, we reparameterize $\lambda=\theta/\gamma$, where $\theta>0$, and define $T:\bR^{SA}\to\bR^{SA}$ by
\begin{equation}\label{e:T.defn}
T(Q)(s,a)=r(s,a)-\frac{\gamma}{\theta}\ln\left[
\sum_{\sp\in\cS}\cP(\sp\mid s,a)
\exp\left(-\theta\max_{\ap\in\cA}Q(\sp,\ap)\right)
\right].
\end{equation}
The operators $F$ and $T$ are related through
\begin{equation}\label{e:F2T}
T(Q)=-\frac{\gamma}{\theta}\ln F\left(\exp\left(-\frac{\theta}{\gamma}Q\right)\right),
\end{equation}
where the operations are applied element-wise. Consequently, $T$ has the unique fixed point $\Qstr=-\frac{\gamma}{\theta}\ln\xstr$. Moreover, Proposition~\ref{prop:TFcontract1} implies that $T$ is a $\gamma$-contraction in the sup norm.

\subsection{Illustrative Example: Risk-averse behavior induced by $T$}
%

The following example shows the qualitative difference between a policy
obtained from the classical risk-neutral Bellman operator and one obtained from
the risk-sensitive formulation represented by $T$.

Consider a discounted MDP with state space $\cS = \{s,\stb\}$, discount factor $\gamma = 0.9$, and risk parameter $\theta = 0.1$. The state $\stb$ is absorbing and yields a reward of $-10$, i.e., there is a single action $a$ such that
\[
    r(\stb,a) = -10, \qquad \cP(\stb \mid \stb,a) = 1.
\]

At state $s$, there are two available actions, $\mathrm{safe}$ and $\mathrm{risk}$. The $\mathrm{safe}$ action yields zero reward and leaves the state unchanged:
\[
    r(s,\mathrm{safe}) = 0, \qquad \cP(s \mid s,\mathrm{safe}) = 1.
\]
In contrast, the $\mathrm{risk}$ action yields a reward of $1$ and induces the following transition probabilities:
\[
    r(s,\mathrm{risk}) = 1, \qquad \cP(\stb \mid s,\mathrm{risk}) = 0.01, \qquad \cP(s \mid s,\mathrm{risk}) = 0.99.
\]

Solving the classical discounted risk-neutral Bellman equations shows that the optimal $Q$-values satisfy
\[
    \Qstr_{\mathrm{rn}}(\stb,a) = -100, \qquad 
    \Qstr_{\mathrm{rn}}(s,\mathrm{safe}) \approx 0.826, \qquad 
    \Qstr_{\mathrm{rn}}(s,\mathrm{risk}) \approx 0.917,
\]
where $\mathrm{rn}$ denotes risk neutral. Thus, the optimal risk-neutral policy selects the $\mathrm{risk}$ action at state $s$.

In contrast, the fixed point $\Qstr$ of our operator $T$ satisfies
\[
    Q^{*}(\stb,a) = -100, \qquad 
    Q^{*}(s,\mathrm{safe}) = 0, \qquad 
    Q^{*}(s,\mathrm{risk}) \approx -47.59.
\]
Thus, in this case, the optimal policy selects the $\mathrm{safe}$ action at state $s$.

In summary, while the risk-neutral formulation favors the $\mathrm{risk}$ action due to its higher expected return, the operator $T$ assigns it a substantially lower value by penalizing the possibility of rare but severe negative outcomes.

\section{Model-free algorithms for optimizing exponential utility}
The previous section motivated the fixed points of the $T$ and $F$ operators.
In this section, we introduce two model-free algorithms for finding them.

Observe that, by treating $\sp$ as a random variable, the summations inside the definitions \eqref{e:T.defn} and \eqref{e:F.defn} of $T$ and $F$,
respectively, can be viewed as conditional expectations. The key difference is
where this expectation appears. In the definition of $T$, it appears inside the
$\log$ function; hence, unlike classical Q-learning, a single transition
sample does not directly yield a stochastic estimate of the full operator value.
In contrast, in the definition of $F$, the expectation appears linearly at the
outermost level, so a single transition sample gives a natural stochastic
estimate of $F(x)$. Keeping this distinction in mind, we now propose two
model-free algorithms to find the fixed points of $T$ and
$F$.

Let $\{(s_n, a_n)\}_{n \geq 0}$ be a sequence of state-action pairs obtained using a (possibly time-varying) behavior policy. Specifically, let $(s_0, a_0)$ be an arbitrary state-action pair and, for $n \geq 0,$ suppose $s_{n + 1} \sim \cP(\cdot|s_n, a_n)$ and $a_{n + 1} \sim \pi_n(\cdot|s_{n + 1}),$ where $\pi_n$ is the behavior policy at time instance $n.$

\subsection{Two-timescale algorithm for estimating $\Qstr$}
The first algorithm is motivated by classical Q-learning. Since $\Qstr$ is the
unique fixed point of $T$, a natural approach is to perform a stochastic
fixed-point iteration for $T$. The main difference from the risk-neutral case is
that the conditional expectation in $T$ lies inside a logarithm. To handle this,
we introduce an auxiliary iterate $g_n$ that tracks this inner expectation on a
faster timescale. This leads to
\begin{equation}
\label{e:2TS.algorithm.update}
    \begin{aligned}
        Q_{n + 1}
        & = Q_n + \alpha_n 
        \left[-\frac{\gamma}{\theta} \ln g_n
        - Q_n\right], \\
        g_{n + 1}
        & = g_n  + \beta_n e_{s_n, a_n}
        [\hG_n - g_n(s_n, a_n)],
    \end{aligned}
\end{equation}
where $\hG_n = G(Q_n,s_n,a_n,s_{n+1})$ and
\begin{equation}
\label{e:G.defn}
    G(Q,s,a,\sp)
    =
    \exp\left[
        -\frac{\theta}{\gamma} r(s,a)
        - \theta \max_{\ap} Q(\sp,\ap)
    \right].
\end{equation}
Here, $e_{s,a}$ denotes the standard basis vector in $\bR^{SA}$ whose
$(s,a)$-th entry is $1$. The faster recursion for $g_n$ estimates the
conditional expectation inside $T(Q_n)$, while the slower recursion for $Q_n$
uses $\log g_n$ as a plug-in estimate and moves toward $\Qstr$.

\subsection{One-timescale algorithm for estimating $\xstr$}
The second algorithm is motivated by the $F$-operator. Since $\xstr$ is the
unique fixed point of $F$, and since the conditional expectation in $F$ appears
linearly at the outermost level, a single transition sample provides a natural
stochastic estimate of $F(x_n)(s_n,a_n)$. This suggests the direct
one-timescale recursion
\begin{equation}
\label{e:1TS.algorithm.update}
    x_{n + 1}
    =
    x_n + \alpha_n e_{s_n, a_n}
    [\hF_n - x_n(s_n, a_n)],
\end{equation}
where
\begin{equation}
\label{e:Fhat.defn}
    \hF_n
    =
    \exp\left(-\frac{\theta}{\gamma} r(s_n,a_n)\right)
    \left[\min_{\ap} x_n(s_{n+1},\ap)\right]^\gamma .
\end{equation}
This update should be contrasted with the sample-based multiplicative update
suggested by the sup-log contraction of $F$: 
\[
    x_{n+1}
    =
    x_n^{1-\alpha_n} \hF_n^{\alpha_n},
\]
where powers and products are understood componentwise. Although this update
is natural from the sup-log viewpoint, it is not a stochastic approximation to
the corresponding logarithmic fixed-point iteration. Indeed, in logarithmic
coordinates it involves $\log \hF_n$, whereas the desired drift contains
$\log F(x_n)$; in general,
\[
    \mathbb{E}\!\left[\log \hF_n \mid s_n,a_n\right]
    \neq
    \log F(x_n)(s_n,a_n).
\]
Thus, exploiting the sup-log contraction of $F$ in a model-free manner would
again require estimating the conditional expectation inside $F$ before taking
the logarithm, leading naturally to an additional tracking recursion. The
additive update \eqref{e:1TS.algorithm.update} avoids this issue by working
directly in the original $x$-coordinates, yielding a genuine one-timescale
algorithm. Once $x_n$ is close to $\xstr$, an estimate of $\Qstr$ is
obtained through the transformation
$\Qstr = -\frac{\gamma}{\theta}\ln \xstr$. The price for this simplicity is
analytical: convergence can no longer be obtained by directly invoking the
sup-log contraction property of $F$.

\begin{remark}
From a computational perspective, the one-timescale algorithm is preferable.
Each iteration requires $\mathcal{O}(1)$ work, whereas the two-timescale
method requires $\mathcal{O}(|\mathcal{S}||\mathcal{A}|)$ work because of
the $Q_n$-update. The one-timescale method also has fewer hyperparameters:
it uses a single step size sequence, while the two-timescale method
requires separate step sizes for the $Q_n$- and $g_n$-updates.
The per-iteration cost of the two-timescale method could be reduced to
$\mathcal{O}(1)$ by updating only one coordinate of $Q_n$ at each iteration.
Beyond these computational considerations, the two-timescale method is a valuable addition to the literature since it is a direct algorithmic realization of
the underlying Bellman equation $TQ = Q$ studied by \cite{porteus1975optimality,chung1987discounted}. Moreover, $T$ is monotone and a contraction in the
$\|\cdot\|_\infty$ norm, closely paralleling the standard risk-neutral
Bellman operator. This structural similarity makes the method
particularly attractive for extending existing ideas from tabular RL
to function approximation.
\end{remark}
\section{Almost-sure Convergence and Finite-time Convergence Rates}
We now discuss the convergence properties of the two proposed algorithms.
The two-timescale recursion inherits a global contraction structure from
the $T$-operator, enabling a relatively standard stochastic approximation
analysis based on timescale separation. In contrast, the one-timescale recursion is driven by the power-law operator $F$. Although
$F$ is contractive in the log-sup metric, this contraction is not directly
aligned with the additive update in the original $x$-coordinates. Consequently,
standard global contraction arguments no longer apply, making both the
asymptotic and finite-time analysis substantially more delicate.

\subsection{Analysis of our two-timescale fixed-point method}
Consider $(Q_n,g_n)_{n \geq 0}$ generated by
\eqref{e:2TS.algorithm.update}. Let
$    \Cm := \frac{\|r\|_\infty}{1-\gamma}$,
where
$\|r\|_\infty := \max_{s,a}|r(s,a)|$.

\begin{assumption}[\textbf{\textit{Controlled Markovian sampling}}]
\label{ass:controlled.chain}
Fix $Q \in \bR^{SA}$ with $\|Q\|_\infty \leq \Cm$ and let the iterates
be frozen, i.e., $Q_n \equiv Q$ $\forall  n \geq 0$. Then the resulting
$(s_n,a_n)_{n\geq0}$ sequence is an irreducible Markov chain on
$\cS\times\cA$ with transition kernel $K_Q$ and unique stationary distribution
$\pi_Q$. Also:

\begin{enumerate}[leftmargin=*]
    \item (\textbf{Uniform ergodicity})
    There exist constants $C<\infty$ and $\rho\in(0,1)$ such that
    \[
        \sup_{y\in\cS\times\cA}
        \|
            K_Q^\ell(\cdot|y)-\pi_Q(\cdot)
        \|_{\rm TV}
        \leq
        C\rho^\ell,
        \qquad
        \forall \ell\geq0,
    \]
    uniformly over all $Q$ satisfying $\|Q\|_\infty \leq \Cm$.

    \item (\textbf{Continuity})
    If $Q_n\to Q$ with
    $\|Q_n\|_\infty,\|Q\|_\infty\leq \Cm$, then
    \[
        K_{Q_n}(y,y')
        \to
        K_Q(y,y')
    \]
    uniformly over $(y,y')\in(\cS\times\cA)^2$.
\end{enumerate}
\end{assumption}

This is a standard assumption in the convergence analysis of stochastic approximation with controlled Markov noise, cf. Condition H2 (b) of \cite{fort2016convergence}.

The next result establishes almost-sure convergence of the two-timescale
iterates. The proof combines timescale separation with the global
$\|\cdot\|_\infty$-contraction of the $T$-operator.

\begin{theorem}[\textbf{\textit{Two-timescale asymptotic convergence}}]
\label{thm:twoTSconv}
Let
\[
    \max\left\{
        \|Q_0\|_\infty,
        \frac{\gamma}{\theta}\|\ln g_0\|_\infty
    \right\}
    \leq \Cm,
\]
and suppose
\[
    \sum_n \alpha_n
    =
    \sum_n \beta_n
    =
    \infty,
    \qquad
    \sum_n (\alpha_n^2+\beta_n^2)<\infty,
    \qquad
    \frac{\alpha_n}{\beta_n}\to0.
\]
Further suppose Assumption~\ref{ass:controlled.chain} holds. Then
\[
    Q_n \overset{\mathrm{a.s.}}{\longrightarrow} \Qstr,
    \qquad
    \left\|
        -\frac{\gamma}{\theta}\ln g_n
        -
        TQ_n
    \right\|_\infty
    \overset{\mathrm{a.s.}}{\longrightarrow}
    0.
\]
\end{theorem}

We next establish finite-time convergence rates in the generative-model
setting, where $(s_n,a_n)$ is sampled i.i.d.\ from a distribution
$\nu$ satisfying $\lambda := \min_{(s,a)\in\cS\times\cA}\nu(s,a)>0$.
This allows us to focus on the impact of the nonlinear $T$-operator and the
two-timescale coupling, while avoiding additional technicalities arising from
Markovian noise, which can typically be handled using standard techniques
\citep{srikant19td, zhang2021finite}.

\begin{theorem}[\textbf{\textit{Two-timescale finite-time rates}}]
\label{thm:2TS.Rates}
In the generative-model setting described above, let
$\alpha_n=(n+1)^{-\alpha}$ and $\beta_n=(n+1)^{-\beta}$, where
$\frac12<\beta<\alpha<1$. Then there exist constants
$C_g,C_Q^{(1)},C_Q^{(2)}>0$ such that
\[
    \bE\left\|
        g_n-e^{-\frac{\theta}{\gamma}TQ_n}
    \right\|_\infty
    \leq
    C_g\sqrt{\beta_n},
\]
and
\[
    \bE\|Q_n-\Qstr\|_\infty
    \leq
    C_Q^{(1)}
    \exp\!\left(
        -\frac{1-\gamma}{1-\alpha}n^{1-\alpha}
    \right)
    +
    C_Q^{(2)}\sqrt{\beta_n}
    =
    O(n^{-\beta/2}).
\]
The precise constant definitions can be found in the appendix.
\end{theorem}

By choosing $\beta$ close to $1$, the rate can be made arbitrarily close to the optimal $O(n^{-1/2})$ rate.

\subsection{Analysis of our one-timescale fixed-point method}
We now state the convergence guarantee for \eqref{e:1TS.algorithm.update}.
Let $\cK := [C_\ell,C_u]^{SA},$ where
\begin{equation}
    \label{e:Cl.Cu.defn}
    C_\ell := \exp\!\left(-\frac{\theta}{\gamma(1-\gamma)}\|r\|_\infty\right)
    \qquad \text{and} \qquad
    C_u := \exp\!\left(\frac{\theta}{\gamma(1-\gamma)}\|r\|_\infty\right)
    .
\end{equation}

\begin{assumption}[\textbf{\textit{Markovian sampling}}]
\label{ass:1TS.markov}
For each frozen $x\in\cK$, the process $(s_n,a_n)$ is an irreducible Markov chain on
$\cS\times\cA$ with kernel $K_x$ and unique stationary distribution $\pi_x$.
The family $\{K_x:x\in\cK\}$ is uniformly ergodic, i.e., for some
$C<\infty$ and $\rho\in(0,1)$,
\[
    \sup_y \|K_x^\ell(\cdot\mid y)-\pi_x(\cdot)\|_{\rm TV}
    \leq C\rho^\ell,
    \qquad \forall \ell\geq0,\ x\in\cK.
\]
Moreover, $x\mapsto K_x$ is Lipschitz continuous on $\cK$, uniformly over entries.
\end{assumption}

\begin{theorem}[\textbf{\textit{One-timescale asymptotic convergence}}]
\label{thm:oneTSconv}
Let Assumption~\ref{ass:1TS.markov} hold. If $x_0\in\cK$ and
$\alpha_n\in[0,1]$ satisfies
$\sum_n\alpha_n=\infty$ and $\sum_n\alpha_n^2<\infty$, then $(x_n)$ 
generated by \eqref{e:1TS.algorithm.update} satisfies
\[
    x_n \to \xstr \textnormal{a.s.}
\]
\end{theorem}

We also give a finite-time result for a scalar stochastic recursion associated with the $x_n$ update in
\eqref{e:1TS.algorithm.update}. This result is not a
finite-time bound for the full Markovian recursion; rather, it shows that the power-law structure
itself does not preclude the standard stochastic-approximation rate.

\begin{proposition}[\textbf{\textit{Scalar rate for the one-timescale update}}]
\label{prop:1TS.scalar.rate}
Consider the scalar stochastic recursion associated with
\eqref{e:1TS.algorithm.update}, as defined in
\eqref{e:1TS.algorithm.update.scalar}, whose mean drift is induced by
$
F(x)= (x/\xstr)^\gamma \xstr.
$ Further suppose $x_0 \in [C_\ell,C_u]$.  Define $\alpha_n=1/(2C_1(n+1)),$ where
\[
    C_1 :=
    \begin{cases}
    \dfrac{\underline y-\underline y^\gamma}{\underline y-1},
    & \underline y\neq 1,\\[1ex]
    1-\gamma, & \underline y=1,
    \end{cases} \qquad \text{and} \qquad \underline y:=\frac{C_\ell}{\xstr}. 
\]
Then, we have $    \bE\left|\dfrac{x_n}{\xstr}-1\right|
    =
    \tilde O(n^{-1/2}),$ where $\tilde{O}$ hides logarithmic terms.
\end{proposition}

\section{Outline of the Proofs}
Here we discuss the main ideas underlying our structural and convergence analysis. We first establish the contraction, monotonicity, and homogeneity properties of the operators $F$ and $T$, and use them to characterize an optimal stationary policy through their fixed points. We then discuss the convergence analysis of our proposed algorithms. Note that the two-timescale method leverages the global $\|\cdot\|_\infty$ contraction of $T$, enabling a relatively standard stochastic approximation analysis via timescale separation. In contrast, the one-timescale recursion is driven by the nonlinear power-law operator $F$, whose contraction holds only in the sup-log metric and is not directly compatible with additive updates in the original coordinates. 

\subsection{Structural properties and optimality characterization}
The proof first establishes that the operators $F$ and $F_\pi$ are $\gamma$-contractions in the sup-log metric by converting multiplicative perturbations into additive logarithmic bounds. The key observation is that if two vectors are close in logarithmic scale, then their componentwise ratios are uniformly controlled, and this control is preserved under the operations defining the operators. Monotonicity and homogeneity follow directly from their definitions. Next, using Jensen's inequality, we show that $X_\pi(\lambda)\preceq F_\pi(X_\pi(\lambda))$. Repeated application of $F_\pi$, together with its monotonicity and contraction properties, then yields $X_\pi(\lambda)\preceq\bar x_\pi$. Finally, the greedy policy $\pi^*$ satisfies $F_{\pi^*}(x^*)=F(x^*)=x^*$, implying $x^*=\bar x_{\pi^*}$. Since $F(x)\preceq F_\pi(x)$ for every stationary policy $\pi$, monotonicity and fixed-point convergence further give $x^*\preceq\bar x_\pi$, establishing the optimality of $\pi^*$ for the recursive formulation.

\subsection{Two-timescale analysis}
The analysis of the two-timescale method is comparatively more standard due to the global $\|\cdot\|_\infty$ contraction of the operator $T$. The key idea is to exploit the separation between the slower $(Q_n)$ recursion and the faster $(g_n)$ recursion. Since
\[
\frac{\alpha_n}{\beta_n}\to0,
\]
the slower iterate $Q_n$ appears nearly static from the perspective of the faster recursion. Consequently, conditioned on a quasi-static $Q_n$, the faster iterate $g_n$ tracks the equilibrium
\[
g^*(Q_n)=\exp\left(-\frac{\theta}{\gamma}TQ_n\right).
\]

To formalize this intuition, we define the tracking error
\[
e_n:=g_n-g^*(Q_n),
\]
and analyze its dynamics on the faster timescale. Using the linear structure of the $g$-update together with standard martingale arguments, we show that
\[
\|e_n\|_\infty \to 0
\qquad \text{a.s.}
\]
Substituting this asymptotic relation into the slower recursion yields
\[
Q_{n+1}
=
Q_n+\alpha_n[TQ_n-Q_n+\epsp_n],
\]
where $\epsp_n$ is a logarithmic transformation of $e_n$. Since $\epsp_n\to0$, the limiting ODE associated with $(Q_n)$ is
\[
\dot Q(t)=TQ(t)-Q(t).
\]

Since $T$ is a $\gamma$-contraction in the $\|\cdot\|_\infty$ norm, the above ODE admits a unique globally asymptotically stable equilibrium $Q^*$. Standard stochastic approximation arguments \citep{borkar2008stochastic} therefore imply
\[
Q_n \to Q^*
\qquad \text{a.s.}
\]

For the finite-time analysis, we again exploit timescale separation. In particular, we first derive a tracking bound for the faster recursion of the form
\[
\bE\|g_n-g^*(Q_n)\|_\infty
=
O(\sqrt{\beta_n}),
\]
and then propagate this error into the slower recursion. Combining the resulting perturbation bounds with the global contraction of $T$ yields the finite-time convergence rates for $(Q_n)$.

\subsection{One-timescale analysis}

The one-timescale recursion is significantly more delicate to analyze. 
Unlike the two-timescale method, the update is driven by the nonlinear power-law operator $F$, whose contraction holds only in the sup-log metric and is therefore not directly compatible with additive stochastic approximation updates in the original coordinates. Moreover, the map $x\mapsto x^\gamma$, $\gamma\in(0,1)$, is not globally Lipschitz near the boundary of $\mathbb{R}^{SA}_{++}$, preventing a direct application of standard global contraction arguments.

We first establish stability of the iterates by showing that the recursion remains within a compact subset
\[
\mathcal K:=[C_\ell,C_u]^{SA}\subset \mathbb{R}^{SA}_{++},
\]
bounded away from the origin. This allows us to show that $F$ is Lipschitz on a suitable neighborhood of $\mathcal K$ and that the limiting ODE
\[
\dot x(t)=D_{\pi_{x(t)}}[F(x(t))-x(t)]
\]
is well posed. We then show that $\mathcal K$ is positively invariant under the above ODE.  

Next, we establish global asymptotic stability of $x^*$ for the above ODE. Since standard quadratic Lyapunov arguments are unavailable, we instead exploit the monotonicity and homogeneity properties of $F$ from Proposition~\ref{prop:TFcontract1}. In particular, homogeneity implies $F(cx)=c^\gamma F(x),$ while monotonicity gives
\[
x\preceq y \implies F(x)\preceq F(y).
\]
Using these properties, we compare the trajectory $x(t)$ with scaled versions of the fixed point $x^*$. Defining
\[
m(t):=\min_i \frac{x_i(t)}{x_i^*},
\qquad \text{and} \qquad
M(t):=\max_i \frac{x_i(t)}{x_i^*},
\]
we obtain
\[
m(t)x^* \preceq x(t)\preceq M(t)x^*.
\]
Applying monotonicity and homogeneity to the left inequality yields
\[
F(x(t))
\succeq
F(m(t)x^*)
=
m(t)^\gamma F(x^*)
=
m(t)^\gamma x^*.
\]
Using this relation, the ODE dynamics, and Danskin's theorem,
we obtain the following bounds on the lower right Dini
derivative of $m(t)$ and the upper right Dini derivative
of $M(t)$:
\[
D_+m(t)\geq
\begin{cases}
\ell_\cK[m(t)^\gamma-m(t)], & m(t)\leq1,\\
u_\cK[m(t)^\gamma-m(t)], & m(t)>1,
\end{cases}
\]
and
\[
D^+M(t)\leq
\begin{cases}
u_\cK[M(t)^\gamma-M(t)], & M(t)\leq1,\\
\ell_\cK[M(t)^\gamma-M(t)], & M(t)>1,
\end{cases}
\]
where $0<\ell_\cK\leq u_\cK$ are uniform bounds on the
stationary probabilities of the state-action pairs.

We then compare these inequalities with the scalar dynamics $\dot z(t)=z(t)^\gamma-z(t), $ whose unique globally asymptotically stable equilibrium is $1$. The comparison argument implies
\[
1\leq\liminf_{t\to\infty}m(t)
\leq\limsup_{t\to\infty}M(t)\leq1.
\]
Consequently, both $m(t)$ and $M(t)$ converge to $1$,
which yields $x(t)\to\xstr$.    
Standard stochastic approximation arguments \citep{borkar2008stochastic} now allow us to show $x_n \to \xstr$ a.s.

We also provide a scalar finite-time analysis that isolates the effect of the power-law drift and highlights the main obstacle in extending standard stochastic approximation rate arguments to the full recursion. In particular, unlike the two-timescale method, the one-timescale recursion does not admit a directly usable global contraction inequality. Nevertheless, by carefully exploiting the structure of the scalar drift, we establish an $\tilde O(n^{-1/2})$ convergence rate.

\section{Conclusions}
\label{sec:conclusions}
Inspired by \citet{porteus1975optimality}, we studied a recursive exponential-utility formulation for discounted MDPs and developed two model-free, Q-learning--style algorithms for learning its optimal stationary policy. We established contraction properties of the associated Bellman operators and showed that their fixed points characterize an optimal stationary policy for the recursive formulation while providing upper bounds on direct exponential utility. We also demonstrated that the optimal stationary policies for the two criteria need not coincide, although their optimal values become close under additional assumptions. Our two algorithms operate on two timescales and one timescale, respectively, and both enjoy almost-sure convergence guarantees under Markovian sampling. We further established finite-time convergence rates for the two-timescale algorithm under i.i.d. sampling and for a scalar recursion associated with the one-timescale method. Extending the latter rate guarantees to general MDPs and developing model-free methods that directly optimize exponential utility beyond the stationary policy class remain interesting directions for future research. 

\begin{ack}
We thank the anonymous reviewers for their valuable comments and suggestions, which helped improve the presentation of this work. Gugan Thoppe acknowledges support from the ANRF ARG Grant (ANRF/ARG/2025/011952/ENS), the Walmart Centre for Tech Excellence at IISc, the Indo-French Centre for the Promotion of Advanced Research (CEFIPRA; Project 7102-1), the Kotak-IISc AIML Centre (KIAC), and the Brain Computation and Data Science (BCD) Centre at IISc. L. A. Prashanth acknowledges support from the A. Raghunathan Center for Theoretical Computer Science and partial support from ANRF under Grant ANRF/ARG/2025/005386/ENS. 
\end{ack}

\bibliographystyle{plainnat}
\bibliography{references}

\clearpage\newpage
\appendix

\section{Structural Properties of $F$ and $F_\pi$}
\label{sec:proofs}

\subsection{Proof of Proposition \ref{prop:TFcontract1}}
\label{sec:TFcontract-proof}

\begin{proof}
    We first prove statements 1 and 2. Fix $x_1,x_2\in\bRp^{SA}$ and let $\delta=d(x_1,x_2)$. By definition,
    \[
e^{-\delta}x_2(s,a)\leq x_1(s,a)\leq e^\delta x_2(s,a),
\qquad \forall (s,a)\in\cS\times\cA.
\]
Taking the minimum over actions, raising to the power $\gamma$, and using the definition of $F$ gives
\[
e^{-\gamma\delta}F(x_2)\preceq F(x_1)
\preceq e^{\gamma\delta}F(x_2).
\]
Taking logarithms and the sup norm yields
\[
d(F(x_1),F(x_2))\leq\gamma\delta
=\gamma d(x_1,x_2),
\]
proving statement 1. Similarly, raising the coordinatewise bounds to the power $\gamma$ and taking the weighted sums appearing in $F_\pi$ gives
\[
e^{-\gamma\delta}F_\pi(x_2)\preceq F_\pi(x_1)
\preceq e^{\gamma\delta}F_\pi(x_2).
\]
The same argument proves statement 2.

For statement 3, note that both $F$ and $F_\pi$ are monotone because their dependence on $x$ involves only order-preserving operations.

To prove statement 4, fix $x\in\bRp^{SA}$. Since a minimum does not exceed a weighted average, $F(x)\preceq F_\pi(x)$ for every $\pi\in\Stat$. Now choose a deterministic stationary policy $\pi_x$ such that
\[
\pi_x(s)\in\argmin_{a\in\cA}x(s,a),
\qquad \forall s\in\cS.
\]
By construction, $F_{\pi_x}(x)=F(x)$, establishing statement 4.

Finally, statement 5 follows directly from the definition of $F$. For every $c>0$ and $x\in\bRp^{SA}$,
\[
F(cx)=c^\gamma F(x).
\]
This completes the proof. 
\end{proof}

\subsection{Proof of Theorem~\ref{thm:motiv}}

\begin{proof}
Fix $\pi\in\Stat$ and $(s,a)\in\cS\times\cA$. Let
\[
Z_1=\exp\left(
-\lambda\sum_{t=1}^{\infty}\gamma^{t-1}r(s_t,a_t)
\right).
\]
By the definition of direct exponential utility,
\[
X_\pi(s,a;\lambda)
=e^{-\lambda r(s,a)}
\bE_{s,a}^{\pi}[Z_1^\gamma].
\]
Since $z\mapsto z^\gamma$ is concave, conditional
Jensen's inequality gives
\begin{align*}
X_\pi(s,a;\lambda)
&\leq e^{-\lambda r(s,a)}
\bE_{s,a}^{\pi}\left[
\left(
\bE_{s,a}^{\pi}[Z_1\mid s_1,a_1]
\right)^\gamma
\right]\\
&=e^{-\lambda r(s,a)}
\bE_{s,a}^{\pi}\left[
X_\pi(s_1,a_1;\lambda)^\gamma
\right]\\
&=F_\pi(X_\pi(\lambda))(s,a),
\end{align*}
where the second equality follows from the Markov
property and stationarity of $\pi$.

Since $(s,a)$ was arbitrary,
$X_\pi(\lambda)\preceq F_\pi(X_\pi(\lambda))$.
By the monotonicity of $F_\pi$, this implies
\[
X_\pi(\lambda)\preceq
F_\pi^k(X_\pi(\lambda)),\qquad k\geq 1.
\]
Proposition~\ref{prop:TFcontract1} guarantees that
$F_\pi$ is a contraction. Hence,
$F_\pi^k(X_\pi(\lambda))\to\bar x_\pi$
as $k\to\infty$, giving
\[
X_\pi(\lambda)\preceq\bar x_\pi.
\]

Next, by the definition of the greedy policy $\pi^*$,
\[
F_{\pi^*}(x^*)=F(x^*)=x^*.
\]
Since $F_{\pi^*}$ has a unique fixed point,
we obtain $x^*=\bar x_{\pi^*}$.

Finally, statement 4 of
Proposition~\ref{prop:TFcontract1} gives
\[
x^*=F(x^*)\preceq F_\pi(x^*).
\]
Using the monotonicity of $F_\pi$ again,
\[
x^*\preceq F_\pi^k(x^*),\qquad k\geq 1.
\]
Letting $k\to\infty$ and using the contraction
property yields $x^*\preceq\bar x_\pi$.
Therefore,
\[
X_\pi(\lambda)\preceq\bar x_\pi,
\qquad
x^*=\bar x_{\pi^*}\preceq\bar x_\pi,
\]
as required.
\end{proof}

\subsection{Counterexample: A stationary policy minimizing
$\bar{x}_{\pi,\gamma}$ need not minimize $X_{\pi,\gamma}$}
\label{sec:counterexample-bar-x-X}

We now show that optimizing the recursively defined criterion
$\bar{x}_{\pi,\gamma}$ need not yield a stationary policy that
optimizes the direct exponential criterion $X_{\pi,\gamma}$. This
mismatch can occur even in a finite MDP for which every stationary
policy induces an irreducible Markov chain.

We first state the main result. Its proof is divided into several
lemmas, each addressing a separate part of the construction.

\begin{theorem}[Mismatch between the two exponential criteria]
\label{thm:counterexample-bar-x-X}
There exist a finite MDP, two deterministic stationary policies
$\pi_a$ and $\pi_b$, and a constant $\gamma_0\in(0,1)$
such that the following statements hold.
\begin{enumerate}[leftmargin=*]
    \item For every $\gamma\in[0,1)$,
    \[
        \bar{x}_{\pi_b,\gamma}
        \preceq
        \bar{x}_{\pi,\gamma},
        \qquad
        \forall\pi\in\Pi_{\mathrm{st}}.
    \]
    Thus, $\pi_b$ minimizes $\bar{x}_{\pi,\gamma}$ componentwise over
    the class of stationary policies.

    \item For every $\gamma\in(\gamma_0,1)$,
    \[
        X_{\pi_a,\gamma}
        \prec
        X_{\pi_b,\gamma}.
    \]
    In fact, for every admissible state-action pair
    $(\tilde s,\tilde a)$,
    \[
        \frac{
            X_{\pi_a,\gamma}(\tilde s,\tilde a)
        }{
            X_{\pi_b,\gamma}(\tilde s,\tilde a)
        }
        \longrightarrow 0
        \qquad
        \text{as }\gamma\uparrow1.
    \]
    Thus, $\pi_b$ becomes arbitrarily worse than $\pi_a$ under the direct exponential criterion.
\end{enumerate}
Consequently, a stationary policy minimizing
$\bar{x}_{\pi,\gamma}$ need not minimize $X_{\pi,\gamma}$.
\end{theorem}

\paragraph{Construction of the MDP.}

Fix the risk parameter $\widehat{\theta}=1$, and consider the state
space
\begin{equation*}
    \mathcal{S}
    =
    \{s,s_h,s_\ell\},
\end{equation*}
where $s$ is the decision state, $s_h$ is the high-reward state, and
$s_\ell$ is the low-reward state. At state $s$, the available actions
are $a$ and $b$. At $s_h$ and $s_\ell$, only action $a$ is available.

The rewards at the decision state are
\begin{equation*}
    r(s,a)=0,
    \qquad
    r(s,b)=\frac{1}{2},
\end{equation*}
while the rewards at the remaining states are
\begin{equation*}
    r(s_h,a)=4,
    \qquad
    r(s_\ell,a)=0.
\end{equation*}

From $s_h$ and $s_\ell$, the process returns deterministically to $s$:
\begin{equation*}
    P(s\mid s_h,a)
    =
    P(s\mid s_\ell,a)
    =
    1.
\end{equation*}
At state $s$, action $a$ induces the transition probabilities
\begin{equation*}
    P(s_h\mid s,a)=\frac{1}{2},
    \qquad
    P(s_\ell\mid s,a)=\frac{1}{2},
\end{equation*}
whereas action $b$ induces
\begin{equation*}
    P(s_h\mid s,b)=\frac{1}{4},
    \qquad
    P(s_\ell\mid s,b)=\frac{3}{4}.
\end{equation*}

Let $\pi_a$ denote the deterministic stationary policy that chooses
action $a$ at $s$, and let $\pi_b$ denote the deterministic stationary
policy that chooses action $b$ at $s$. At $s_h$ and $s_\ell$, both
policies choose the only available action $a$.

The set of admissible state-action pairs is
\begin{equation*}
    \mathcal{S}_{\mathrm{adm}}
    :=
    \{
        (s,a),
        (s,b),
        (s_h,a),
        (s_\ell,a)
    \}.
\end{equation*}

We begin with a structural property of the example.

\begin{lemma}[Irreducibility]
\label{lem:counterexample-irreducibility}
Every stationary policy in the preceding MDP induces an irreducible
Markov chain.
\end{lemma}

\begin{proof}
Under either action at $s$, both $s_h$ and $s_\ell$ are reached with
strictly positive probability. The same is therefore true under every
randomized stationary policy. Moreover, both $s_h$ and $s_\ell$
return deterministically to $s$. Hence, every state communicates with
every other state under every stationary policy.
\end{proof}

We next determine the policy that minimizes the recursively defined
criterion $\bar{x}_{\pi,\gamma}$.

\begin{lemma}[Optimality of $\pi_b$ for the recursive criterion]
\label{lem:pi-b-bar-x-optimal}
For every $\gamma\in[0,1)$, the policy $\pi_b$ is greedy with respect
to the unique fixed point $x_\gamma^*$ of the optimality operator
$F_\gamma$. Consequently,
\begin{equation*}
    \bar{x}_{\pi_b,\gamma}
    =
    x_\gamma^*
    \preceq
    \bar{x}_{\pi,\gamma},
    \qquad
    \forall \pi\in\Pi_{\mathrm{st}}.
\end{equation*}
\end{lemma}

\begin{proof}
Fix $\gamma\in[0,1)$, and define
\begin{equation*}
    m_\gamma^*
    :=
    \min
    \left\{
        x_\gamma^*(s,a),
        x_\gamma^*(s,b)
    \right\}.
\end{equation*}
Since action $a$ is the only available action at $s_h$ and $s_\ell$,
the fixed-point relation
\begin{equation*}
    F_\gamma(x_\gamma^*)
    =
    x_\gamma^*
\end{equation*}
gives
\begin{equation*}
    x_\gamma^*(s_h,a)
    =
    e^{-4}(m_\gamma^*)^\gamma
\end{equation*}
and
\begin{equation*}
    x_\gamma^*(s_\ell,a)
    =
    (m_\gamma^*)^\gamma.
\end{equation*}

It follows that
\begin{align*}
    x_\gamma^*(s,a)
    &=
    \frac{1}{2}
    \bigl[x_\gamma^*(s_h,a)\bigr]^\gamma
    +
    \frac{1}{2}
    \bigl[x_\gamma^*(s_\ell,a)\bigr]^\gamma
    \\
    &=
    \frac{1+e^{-4\gamma}}{2}
    (m_\gamma^*)^{\gamma^2}.
\end{align*}
Similarly,
\begin{align*}
    x_\gamma^*(s,b)
    &=
    e^{-1/2}
    \left(
        \frac{1}{4}
        \bigl[x_\gamma^*(s_h,a)\bigr]^\gamma
        +
        \frac{3}{4}
        \bigl[x_\gamma^*(s_\ell,a)\bigr]^\gamma
    \right)
    \\
    &=
    e^{-1/2}
    \frac{3+e^{-4\gamma}}{4}
    (m_\gamma^*)^{\gamma^2}.
\end{align*}
Therefore,
\begin{equation*}
    \frac{x_\gamma^*(s,b)}
         {x_\gamma^*(s,a)}
    =
    e^{-1/2}
    \frac{3+e^{-4\gamma}}
         {2(1+e^{-4\gamma})}.
\end{equation*}
For every $\gamma\in[0,1)$,
\begin{equation*}
    \frac{3+e^{-4\gamma}}
         {2(1+e^{-4\gamma})}
    <
    \frac{3}{2}
    <
    e^{1/2}.
\end{equation*}
Hence,
\begin{equation*}
    x_\gamma^*(s,b)
    <
    x_\gamma^*(s,a).
\end{equation*}

Thus, action $b$ is the unique greedy action at $s$, and $\pi_b$ is
greedy with respect to $x_\gamma^*$. It follows that
\begin{equation*}
    F_{\pi_b,\gamma}(x_\gamma^*)
    =
    F_\gamma(x_\gamma^*)
    =
    x_\gamma^*.
\end{equation*}
Since $F_{\pi_b,\gamma}$ has a unique fixed point,
\begin{equation*}
    \bar{x}_{\pi_b,\gamma}
    =
    x_\gamma^*.
\end{equation*}
Finally, Theorem~\ref{thm:motiv} gives
\begin{equation*}
    x_\gamma^*
    \preceq
    \bar{x}_{\pi,\gamma},
    \qquad
    \forall \pi\in\Pi_{\mathrm{st}},
\end{equation*}
which proves the result.
\end{proof}

Lemma~\ref{lem:pi-b-bar-x-optimal} establishes the first part of
Theorem~\ref{thm:counterexample-bar-x-X}. We now analyze the direct
exponential criterion.

For each $i\in\{a,b\}$, define
\begin{equation*}
    c_i:=r(s,i),
\end{equation*}
so that
\begin{equation*}
    c_a=0,
    \qquad
    c_b=\frac{1}{2}.
\end{equation*}
Let $p_i$ denote the probability that action $i$ at $s$ moves the
process to $s_h$. Thus,
\begin{equation*}
    p_a=\frac{1}{2},
    \qquad
    p_b=\frac{1}{4}.
\end{equation*}

Let $R_i$ denote the reward received at the next state-action pair after taking action $i$ at state $s$. Since action $a$ is the only available action at both $s_h$ and $s_\ell$, we have
\[
    R_i=
    \begin{cases}
        r(s_h,a)=4, & \text{with probability } p_i,\\
        r(s_\ell,a)=0, & \text{with probability } 1-p_i.
    \end{cases}
\]
For $u\in[0,1]$, define its Laplace transform by
\begin{equation*}
    M_i(u)
    :=
    \mathbb{E}\left[e^{-uR_i}\right]
    =
    p_i e^{-4u}+(1-p_i).
\end{equation*}
In particular,
\begin{equation*}
    M_a(u)
    =
    \frac{1+e^{-4u}}{2},
    \qquad
    M_b(u)
    =
    \frac{3+e^{-4u}}{4}.
\end{equation*}

Recall that, for a stationary policy $\pi$ and an admissible
state-action pair $(\tilde{s},\tilde{a})$,
\begin{equation*}
    X_{\pi,\gamma}(\tilde{s},\tilde{a})
    :=
    \mathbb{E}
    \left[
        \left.
        \exp
        \left(
            -\sum_{t=0}^{\infty}
            \gamma^t r(S_t,A_t)
        \right)
        \,\right|\,
        S_0=\tilde{s},\,
        A_0=\tilde{a}
    \right].
\end{equation*}
The prescribed action $\tilde{a}$ is used at time zero, and the policy
$\pi$ is followed from time one onward.

The following lemma gives exact product representations for all
components of $X_{\pi_i,\gamma}$.

\begin{lemma}[Exact representation of the direct criterion]
\label{lem:X-exact-representation}
For every $i,j\in\{a,b\}$,
\begin{align}
    X_{\pi_i,\gamma}(s,j)
    &=
    e^{-c_j}M_j(\gamma)
    \exp
    \left(
        -\frac{\gamma^2c_i}{1-\gamma^2}
    \right)
    \prod_{k=0}^{\infty}
    M_i\left(\gamma^{2k+3}\right).
    \label{eq:X-decision-state-new}
\end{align}
Moreover,
\begin{align}
    X_{\pi_i,\gamma}(s_\ell,a)
    &=
    \exp
    \left(
        -\frac{\gamma c_i}{1-\gamma^2}
    \right)
    \prod_{k=0}^{\infty}
    M_i\left(\gamma^{2k+2}\right),
    \label{eq:X-low-state-new}
\end{align}
and
\begin{align}
    X_{\pi_i,\gamma}(s_h,a)
    &=
    e^{-4}
    X_{\pi_i,\gamma}(s_\ell,a).
    \label{eq:X-high-state-new}
\end{align}
\end{lemma}

\begin{proof}
Fix $i,j\in\{a,b\}$. Starting from $(s,j)$ and following $\pi_i$
from time one onward, the reward sequence has the form
\begin{equation*}
    c_j,\;
    R^{(j)},\;
    c_i,\;
    R_0^{(i)},\;
    c_i,\;
    R_1^{(i)},\ldots,
\end{equation*}
where $R^{(j)}$ has the same distribution as $R_j$, and
$(R_k^{(i)})_{k\geq0}$ are independent copies of $R_i$. Furthermore,
$R^{(j)}$ is independent of $(R_k^{(i)})_{k\geq0}$. Therefore,
\begin{align*}
    X_{\pi_i,\gamma}(s,j)
    &=
    \mathbb{E}
    \left[
        \exp
        \left(
            -c_j
            -\gamma R^{(j)}
            -\sum_{k=0}^{\infty}\gamma^{2k+2}c_i
            -\sum_{k=0}^{\infty}
                \gamma^{2k+3}R_k^{(i)}
        \right)
    \right]
    \\
    &=
    e^{-c_j}M_j(\gamma)
    \exp
    \left(
        -c_i
        \sum_{k=0}^{\infty}\gamma^{2k+2}
    \right)
    \prod_{k=0}^{\infty}
    M_i\left(\gamma^{2k+3}\right)
    \\
    &=
    e^{-c_j}M_j(\gamma)
    \exp
    \left(
        -\frac{\gamma^2c_i}{1-\gamma^2}
    \right)
    \prod_{k=0}^{\infty}
    M_i\left(\gamma^{2k+3}\right).
\end{align*}

Starting from $(s_\ell,a)$ and following $\pi_i$ thereafter, the reward
sequence is
\begin{equation*}
    0,\;
    c_i,\;
    R_0^{(i)},\;
    c_i,\;
    R_1^{(i)},\ldots.
\end{equation*}
Consequently,
\begin{align*}
    X_{\pi_i,\gamma}(s_\ell,a)
    &=
    \exp
    \left(
        -c_i
        \sum_{k=0}^{\infty}\gamma^{2k+1}
    \right)
    \prod_{k=0}^{\infty}
    M_i\left(\gamma^{2k+2}\right)
    \\
    &=
    \exp
    \left(
        -\frac{\gamma c_i}{1-\gamma^2}
    \right)
    \prod_{k=0}^{\infty}
    M_i\left(\gamma^{2k+2}\right).
\end{align*}

Finally, starting from $(s_h,a)$ produces the same continuation as
starting from $(s_\ell,a)$, except that the initial reward is $4$
instead of $0$. Hence,
\begin{equation*}
    X_{\pi_i,\gamma}(s_h,a)
    =
    e^{-4}
    X_{\pi_i,\gamma}(s_\ell,a).
\end{equation*}
This completes the proof.
\end{proof}

The infinite products in
Lemma~\ref{lem:X-exact-representation} can be evaluated asymptotically
using the following shifted Riemann-sum result.

\begin{lemma}[Shifted Riemann-sum limit]
\label{lem:shifted-riemann-sum}
Let $M:[0,1]\to(0,\infty)$ be continuous and suppose that there exists
a constant $C<\infty$ such that
\begin{equation*}
    |\log M(u)|
    \leq
    Cu,
    \qquad
    u\in[0,1].
\end{equation*}
Then, for every fixed positive integer $\kappa$,
\begin{equation*}
    \lim_{\gamma\uparrow1}
    (1-\gamma)
    \sum_{k=0}^{\infty}
    \log M\left(\gamma^{2k+\kappa}\right)
    =
    \frac{1}{2}
    \int_0^1
    \frac{\log M(u)}{u}\,du.
\end{equation*}
\end{lemma}

\begin{proof}
Define
\begin{equation*}
    g(t)
    :=
    \log M(e^{-t}),
    \qquad
    t\geq0.
\end{equation*}
The assumed bound implies
\begin{equation*}
    |g(t)|
    \leq
    Ce^{-t},
    \qquad
    t\geq0.
\end{equation*}
Thus, $g$ is continuous and integrable on $[0,\infty)$.

Set
\begin{equation*}
    h:=-\log\gamma.
\end{equation*}
As $\gamma\uparrow1$, we have $h\downarrow0$, and
\begin{equation*}
    \gamma^{2k+\kappa}
    =
    e^{-(2k+\kappa)h}.
\end{equation*}
Therefore,
\begin{equation*}
    2h
    \sum_{k=0}^{\infty}
    \log M\left(\gamma^{2k+\kappa}\right)
    =
    2h
    \sum_{k=0}^{\infty}
    g\left((2k+\kappa)h\right).
\end{equation*}

The expression on the right is a shifted Riemann sum with mesh size
$2h$. Its evaluation points are
\begin{equation*}
    \kappa h,\;
    (\kappa+2)h,\;
    (\kappa+4)h,\ldots.
\end{equation*}
Since $\kappa$ is fixed, the initial shift $\kappa h$ converges to zero
as $h\downarrow0$. The continuity of $g$ on compact intervals,
together with the bound
\begin{equation*}
    |g(t)|
    \leq
    Ce^{-t},
\end{equation*}
for controlling the tail, gives
\begin{equation*}
    2h
    \sum_{k=0}^{\infty}
    g\left((2k+\kappa)h\right)
    \longrightarrow
    \int_0^\infty g(t)\,dt.
\end{equation*}

Using the change of variables $u=e^{-t}$, we obtain
\begin{equation*}
    \int_0^\infty g(t)\,dt
    =
    \int_0^1
    \frac{\log M(u)}{u}\,du.
\end{equation*}
Finally,
\begin{equation*}
    \frac{1-\gamma}{h}
    =
    \frac{1-\gamma}{-\log\gamma}
    \longrightarrow
    1.
\end{equation*}
It follows that
\begin{align*}
    &(1-\gamma)
    \sum_{k=0}^{\infty}
    \log M\left(\gamma^{2k+\kappa}\right)
    \\
    &\qquad=
    \frac{1-\gamma}{2h}
    \left[
        2h
        \sum_{k=0}^{\infty}
        g\left((2k+\kappa)h\right)
    \right]
    \\
    &\qquad\longrightarrow
    \frac{1}{2}
    \int_0^1
    \frac{\log M(u)}{u}\,du,
\end{align*}
as desired.
\end{proof}

We now apply Lemma~\ref{lem:shifted-riemann-sum} to the functions
$M_a$ and $M_b$. For every $i\in\{a,b\}$ and $u\in[0,1]$,
\begin{equation*}
    e^{-4u}
    \leq
    M_i(u)
    \leq
    1.
\end{equation*}
Consequently,
\begin{equation*}
    |\log M_i(u)|
    \leq
    4u,
\end{equation*}
and the conditions of Lemma~\ref{lem:shifted-riemann-sum} are
satisfied.

The next lemma shows that all components associated with a fixed policy
have the same leading-order asymptotic rate.

\begin{lemma}[Common asymptotic rate]
\label{lem:X-common-asymptotic-rate}
For $i\in\{a,b\}$, define
\begin{equation*}
    \Lambda_i
    :=
    \frac{1}{2}
    \left[
        -c_i
        +
        \int_0^1
        \frac{\log M_i(u)}{u}\,du
    \right].
\end{equation*}
Then, for every admissible state-action pair
$(\tilde{s},\tilde{a})\in\mathcal{S}_{\mathrm{adm}}$,
\begin{equation*}
    \lim_{\gamma\uparrow1}
    (1-\gamma)
    \log
    X_{\pi_i,\gamma}(\tilde{s},\tilde{a})
    =
    \Lambda_i.
\end{equation*}
\end{lemma}

\begin{proof}
Consider first an initial pair $(s,j)$, where $j\in\{a,b\}$. By
\eqref{eq:X-decision-state-new},
\begin{align*}
    &(1-\gamma)
    \log X_{\pi_i,\gamma}(s,j)
    \\
    &=
    (1-\gamma)
    \left[
        -c_j+\log M_j(\gamma)
    \right]
    -
    (1-\gamma)
    \frac{\gamma^2c_i}{1-\gamma^2}
    \\
    &\qquad
    +
    (1-\gamma)
    \sum_{k=0}^{\infty}
    \log M_i\left(\gamma^{2k+3}\right).
\end{align*}

The first term converges to zero because
$-c_j+\log M_j(\gamma)$ remains bounded as $\gamma\uparrow1$.
Moreover,
\begin{equation*}
    (1-\gamma)
    \frac{\gamma^2}{1-\gamma^2}
    =
    \frac{\gamma^2}{1+\gamma}
    \longrightarrow
    \frac{1}{2}.
\end{equation*}
Applying Lemma~\ref{lem:shifted-riemann-sum} with $\kappa=3$ gives
\begin{align*}
    \lim_{\gamma\uparrow1}
    (1-\gamma)
    \log X_{\pi_i,\gamma}(s,j)
    &=
    -\frac{c_i}{2}
    +
    \frac{1}{2}
    \int_0^1
    \frac{\log M_i(u)}{u}\,du
    \\
    &=
    \Lambda_i.
\end{align*}

Similarly, by \eqref{eq:X-low-state-new},
\begin{align*}
    &(1-\gamma)
    \log X_{\pi_i,\gamma}(s_\ell,a)
    \\
    &=
    -(1-\gamma)
    \frac{\gamma c_i}{1-\gamma^2}
    +
    (1-\gamma)
    \sum_{k=0}^{\infty}
    \log M_i\left(\gamma^{2k+2}\right).
\end{align*}
Since
\begin{equation*}
    (1-\gamma)
    \frac{\gamma}{1-\gamma^2}
    =
    \frac{\gamma}{1+\gamma}
    \longrightarrow
    \frac{1}{2},
\end{equation*}
Lemma~\ref{lem:shifted-riemann-sum}, applied with $\kappa=2$, yields
\begin{equation*}
    \lim_{\gamma\uparrow1}
    (1-\gamma)
    \log X_{\pi_i,\gamma}(s_\ell,a)
    =
    \Lambda_i.
\end{equation*}

Finally, by \eqref{eq:X-high-state-new},
\begin{equation*}
    (1-\gamma)
    \log X_{\pi_i,\gamma}(s_h,a)
    =
    -4(1-\gamma)
    +
    (1-\gamma)
    \log X_{\pi_i,\gamma}(s_\ell,a).
\end{equation*}
The first term converges to zero, and hence the limit is again
$\Lambda_i$.
\end{proof}

Lemma~\ref{lem:X-common-asymptotic-rate} shows that differences arising
from the initial state-action pair do not affect the leading-order
behavior on the $(1-\gamma)\log$ scale. It remains to compare the rates
$\Lambda_a$ and $\Lambda_b$.

We first establish a pointwise estimate that will be used to bound the
difference between these rates.

\begin{lemma}[A pointwise lower bound]
\label{lem:y-pointwise-lower-bound}
Define
\begin{equation*}
    y(u)
    :=
    \frac{1-e^{-4u}}{3+e^{-4u}},
    \qquad
    u\in[0,1].
\end{equation*}
Then
\begin{equation*}
    y(u)
    \geq
    \frac{u}{1+\frac{9}{4}u},
    \qquad
    u\in[0,1].
\end{equation*}
\end{lemma}

\begin{proof}
The desired inequality is equivalent to
\begin{equation*}
    e^{-4u}
    \leq
    \frac{4-3u}{4+13u}.
\end{equation*}
To verify this inequality, define
\begin{equation*}
    \phi(u)
    :=
    4u
    +
    \log
    \left(
        \frac{4-3u}{4+13u}
    \right).
\end{equation*}
We have $\phi(0)=0$, and a direct calculation gives
\begin{equation*}
    \phi'(u)
    =
    \frac{4u(40-39u)}
         {(4-3u)(4+13u)}
    \geq
    0,
    \qquad
    u\in[0,1].
\end{equation*}
Therefore, $\phi(u)\geq0$ for every $u\in[0,1]$. Equivalently,
\begin{equation*}
    4u
    +
    \log
    \left(
        \frac{4-3u}{4+13u}
    \right)
    \geq
    0,
\end{equation*}
which gives
\begin{equation*}
    e^{-4u}
    \leq
    \frac{4-3u}{4+13u}.
\end{equation*}
This proves the claimed lower bound.
\end{proof}

We can now establish the strict comparison between the two asymptotic
rates.

\begin{lemma}[Strict comparison of asymptotic rates]
\label{lem:strict-rate-comparison}
The rates in Lemma~\ref{lem:X-common-asymptotic-rate} satisfy
\begin{equation*}
    \Lambda_a
    <
    \Lambda_b.
\end{equation*}
\end{lemma}

\begin{proof}
By the definitions of $\Lambda_a$ and $\Lambda_b$,
\begin{align*}
    \Lambda_a-\Lambda_b
    &=
    \frac{1}{2}
    \left[
        c_b-c_a
        +
        \int_0^1
        \frac{
            \log M_a(u)-\log M_b(u)
        }{u}
        \,du
    \right]
    \\
    &=
    \frac{1}{2}
    \left[
        \frac{1}{2}
        +
        \int_0^1
        \frac{1}{u}
        \log
        \left(
            \frac{2(1+e^{-4u})}
                 {3+e^{-4u}}
        \right)
        \,du
    \right].
\end{align*}

Using the function $y$ from
Lemma~\ref{lem:y-pointwise-lower-bound}, we have
\begin{equation*}
    \frac{2(1+e^{-4u})}
         {3+e^{-4u}}
    =
    1-y(u).
\end{equation*}
Since $\log(1-z)\leq-z$ for every $z\in[0,1)$,
\begin{align}
    \int_0^1
    \frac{1}{u}
    \log
    \left(
        \frac{2(1+e^{-4u})}
             {3+e^{-4u}}
    \right)
    \,du
    &\leq
    -
    \int_0^1
    \frac{y(u)}{u}\,du.
    \label{eq:integral-upper-bound-new}
\end{align}

Lemma~\ref{lem:y-pointwise-lower-bound} gives
\begin{align*}
    \int_0^1
    \frac{y(u)}{u}\,du
    &\geq
    \int_0^1
    \frac{1}{1+\frac{9}{4}u}\,du
    \\
    &=
    \frac{4}{9}
    \log
    \left(
        \frac{13}{4}
    \right)
    \\
    &>
    \frac{1}{2}.
\end{align*}
Combining this inequality with
\eqref{eq:integral-upper-bound-new} yields
\begin{equation*}
    \int_0^1
    \frac{1}{u}
    \log
    \left(
        \frac{2(1+e^{-4u})}
             {3+e^{-4u}}
    \right)
    \,du
    <
    -\frac{1}{2}.
\end{equation*}
Consequently,
\begin{equation*}
    \Lambda_a-\Lambda_b
    <
    \frac{1}{2}
    \left[
        \frac{1}{2}-\frac{1}{2}
    \right]
    =
    0.
\end{equation*}
Thus, $\Lambda_a<\Lambda_b$.
\end{proof}

We are now ready to prove the main result.

\begin{proof}[Proof of Theorem~\ref{thm:counterexample-bar-x-X}]
Lemma~\ref{lem:pi-b-bar-x-optimal} shows that, for every
$\gamma\in[0,1)$,
\begin{equation*}
    \bar{x}_{\pi_b,\gamma}
    \preceq
    \bar{x}_{\pi,\gamma},
    \qquad
    \forall\pi\in\Pi_{\mathrm{st}}.
\end{equation*}
Thus, $\pi_b$ minimizes the recursively defined criterion
$\bar{x}_{\pi,\gamma}$ componentwise over all stationary policies.

Now fix an admissible state-action pair
$(\tilde{s},\tilde{a})\in\mathcal\mathcal{S}_{\mathrm{adm}}$. By
Lemmas~\ref{lem:X-common-asymptotic-rate}
and~\ref{lem:strict-rate-comparison},
\begin{align*}
    &\lim_{\gamma\uparrow1}
    (1-\gamma)
    \log
    \left(
        \frac{
            X_{\pi_a,\gamma}(\tilde{s},\tilde{a})
        }{
            X_{\pi_b,\gamma}(\tilde{s},\tilde{a})
        }
    \right)
    \\
    &\qquad=
    \Lambda_a-\Lambda_b
    <
    0.
\end{align*}
Therefore, there exists
$\gamma_{\tilde{s},\tilde{a}}\in(0,1)$ such that
\begin{equation*}
    X_{\pi_a,\gamma}(\tilde{s},\tilde{a})
    <
    X_{\pi_b,\gamma}(\tilde{s},\tilde{a})
\end{equation*}
for every
$\gamma\in(\gamma_{\tilde{s},\tilde{a}},1)$.

Since $\mathcal{S}_{\mathrm{adm}}$ is finite, define
\begin{equation*}
    \gamma_0
    :=
    \max_{(\tilde{s},\tilde{a})\in \mathcal{S}_{\mathrm{adm}}  }
    \gamma_{\tilde{s},\tilde{a}}.
\end{equation*}
Then $\gamma_0<1$, and, for every $\gamma\in(\gamma_0,1)$,
\begin{equation*}
    X_{\pi_a,\gamma}(\tilde{s},\tilde{a})
    <
    X_{\pi_b,\gamma}(\tilde{s},\tilde{a}),
    \qquad
    \forall
    (\tilde{s},\tilde{a})\in \mathcal{S}_{\mathrm{adm}}.
\end{equation*}
Equivalently,
\begin{equation*}
    X_{\pi_a,\gamma}
    \prec
    X_{\pi_b,\gamma}.
\end{equation*}

Thus, although $\pi_b$ minimizes
$\bar{x}_{\pi,\gamma}$ componentwise over the class of stationary
policies, it is strictly dominated by the stationary policy $\pi_a$
under the direct exponential criterion whenever $\gamma$ is
sufficiently close to one. In particular, $\pi_b$ does not minimize
$X_{\pi,\gamma}$ over $\Pi_{\mathrm{st}}$.

Finally, Lemma~\ref{lem:counterexample-irreducibility} shows that this
mismatch occurs even though every stationary policy induces an
irreducible Markov chain.
\end{proof}

\subsection{Asymptotic Agreement in Uniformly Absorbing MDPs}
\label{sec:absorbing-agreement}

The preceding example shows that, in a general MDP, minimizing
$\bar{x}_{\pi,\gamma}$ need not minimize $X_{\pi,\gamma}$, even when
$\gamma$ is arbitrarily close to one. We now show that, in a finite MDP
with an absorbing set and nonnegative rewards,
$\bar{x}_{\pi,\gamma}$ and $X_{\pi,\gamma}$ become uniformly close as
$\gamma\uparrow1$.

Let $\mathcal{S}_{\mathrm{abs}}\subseteq\mathcal{S}$ be a nonempty
absorbing set. We assume that every action available at a state in
$\mathcal{S}_{\mathrm{abs}}$ yields zero reward and keeps the process
within $\mathcal{S}_{\mathrm{abs}}$. Thus, once the process enters
$\mathcal{S}_{\mathrm{abs}}$, all subsequent rewards are zero.

Let
\[
    \mathcal{Z}
    :=
    \left\{
        (s,a):
        s\notin\mathcal{S}_{\mathrm{abs}},
        \ a\in\mathcal{A}(s)
    \right\}
\]
denote the set of nonabsorbing state-action pairs. For
$\pi\in\Stat$, define the policy-induced state-action transition matrix
\[
    P_\pi \bigl((s,a),(s',a')\bigr)
    :=
    P(s'\mid s,a)\pi(a'\mid s'),
\]
and let $Z_\pi$ denote the restriction of $P_\pi$ to $\mathcal{Z}\times\mathcal{Z}$.

We impose the following standard assumption.
\begin{assumption}
\label{ass:all-policies-proper}
The rewards are nonnegative, and every stationary policy
$\pi\in\Stat$ is proper; that is, under $\pi$, the absorbing set
$\mathcal{S}_{\mathrm{abs}}$ is reached almost surely from every
initial state.
\end{assumption}

The following lemma records the uniform-transience consequence of
Assumption~\ref{ass:all-policies-proper}.

\begin{lemma}
\label{lem:uniform-transience}
Under Assumption~\ref{ass:all-policies-proper}, there exists a constant
$K<\infty$ such that
\[
    \sup_{\pi\in\Stat}
    \left\|
        \sum_{n=1}^{\infty}Z_\pi^n
    \right\|_\infty
    \leq K.
\]
\end{lemma}

\begin{proof}
Fix $\pi\in\Stat$. Since $\pi$ is proper, every state-action pair in
$\mathcal{Z}$ is transient under the state-action Markov chain induced
by $\pi$. Hence,
\[
    \rho(Z_\pi)<1,
\]
and therefore
\[
    \sum_{n=1}^{\infty}Z_\pi^n
    =
    (I-Z_\pi)^{-1}-I.
\]
The set $\Stat$ of stationary policies is compact, and the map
$\pi\mapsto Z_\pi$ is continuous. Moreover, $I-Z_\pi$ is nonsingular
for every $\pi\in\Stat$. Consequently, the map
\[
    \pi
    \longmapsto
    (I-Z_\pi)^{-1}-I
\]
is continuous on the compact set $\Stat$ and is therefore uniformly
bounded. This proves the claim.
\end{proof}

For a fixed policy $\pi$, the vector
\[
    \sum_{n=1}^{\infty} Z_\pi^n\ones
\]
gives the expected number of visits, at positive time instances, to
nonabsorbing state-action pairs before absorption. Thus,
Lemma~\ref{lem:uniform-transience} states that the expected time to
absorption is uniformly bounded over stationary policies and initial
nonabsorbing state-action pairs.

For $\gamma=1$, define the undiscounted absorbing criterion by
\[
    X_{\pi,1}(s,a)
    :=
    \mathbb{E}_{s,a}^{\pi}
    \left[
        \exp\left(
            -\widehat{\theta}
            \sum_{t=0}^{\tau-1}r(S_t,A_t)
        \right)
    \right],
\]
where
\[
    \tau
    :=
    \inf\left\{
        t\geq0:
        S_t\in\mathcal{S}_{\mathrm{abs}}
    \right\}
\]
is the first hitting time of the absorbing set. Here, action $a$ is
applied at time zero and policy $\pi$ is followed thereafter. By
Assumption~\ref{ass:all-policies-proper}, $\tau<\infty$ almost surely;
moreover, Lemma~\ref{lem:uniform-transience} gives a uniform bound on its
expectation.

For every state-action pair $(s,a)$ with
$s\in\mathcal{S}_{\mathrm{abs}}$, all subsequent rewards are zero, and
hence
\[
    X_{\pi,1}(s,a)
    =
    X_{\pi,\gamma}(s,a)
    =
    \bar{x}_{\pi,\gamma}(s,a)
    =
    1.
\]

We begin with a simple ordering relation.

\begin{lemma}
\label{lem:absorbing-order}
Under Assumption~\ref{ass:all-policies-proper}, for every
$\pi\in\Stat$ and $\gamma\in(0,1)$,
\[
    \mathbf{0}
    \prec
    X_{\pi,1}
    \preceq
    X_{\pi,\gamma}
    \preceq
    \bar{x}_{\pi,\gamma}
    \preceq
    \ones.
\]
\end{lemma}

\begin{proof}
Fix $\pi\in\Stat$. By
Assumption~\ref{ass:all-policies-proper}, the absorption time $\tau$ is
finite almost surely. Since the MDP is finite, the reward function is
bounded, and hence
\[
    \sum_{t=0}^{\tau-1} r(S_t,A_t)
\]
is finite almost surely. Consequently, the random variable defining
$X_{\pi,1}$ is strictly positive almost surely, which gives
\[
    \mathbf{0}
    \prec
    X_{\pi,1}.
\]

Since the rewards are nonnegative and $\gamma\in(0,1)$, we have, on
every sample path,
\[
    \sum_{t=0}^{\tau-1}\gamma^t r(S_t,A_t)
    \leq
    \sum_{t=0}^{\tau-1}r(S_t,A_t).
\]
Because the map $u\mapsto e^{-\widehat{\theta}u}$ is decreasing, it
follows that
\[
    X_{\pi,1}
    \preceq
    X_{\pi,\gamma}.
\]
The inequality
\[
    X_{\pi,\gamma}
    \preceq
    \bar{x}_{\pi,\gamma}
\]
follows from Theorem~\ref{thm:motiv}.

It remains to show that
$\bar{x}_{\pi,\gamma}\preceq\ones$. This follows from Proposition~\ref{prop:TFcontract1} since the rewards are nonnegative, which implies
\[
    F_{\pi,\gamma}(\ones)
    \preceq
    \ones.
\]
Combining the preceding inequalities proves the result.
\end{proof}

The next lemma gives the matrix estimate needed in the proof.

\begin{lemma}
\label{lem:absorbing-resolvent}
For every $\pi\in\Stat$, define
\[
    D
    :=
    \operatorname{diag}
    \left(
        e^{-\widehat{\theta}r(s,a)}
    \right)_{(s,a)\in\mathcal{Z}}.
\]
Then $I-DZ_\pi$ is invertible and
\[
    (I-DZ_\pi)^{-1}DZ_\pi
    =
    \sum_{n=1}^{\infty}(DZ_\pi)^n.
\]
Moreover,
\[
    \mathbf{0}
    \preceq
    (I-DZ_\pi)^{-1}DZ_\pi
    \preceq
    \sum_{n=1}^{\infty}Z_\pi^n.
\]
In particular,
\[
    \left\|
        (I-DZ_\pi)^{-1}DZ_\pi
    \right\|_\infty
    \leq
    K,
\]
where $K$ is the constant from
Lemma~\ref{lem:uniform-transience}.
\end{lemma}

\begin{proof}
Fix $\pi\in\Stat$. Since the rewards are nonnegative,
\[
    0
    <
    e^{-\widehat{\theta}r(s,a)}
    \leq
    1,
    \qquad
    (s,a)\in\mathcal Z.
\]
Therefore,
\[
    \mathbf{0}
    \preceq
    D
    \preceq
    I.
\]
Since $Z_\pi$ is nonnegative, this implies
\[
    \mathbf{0}
    \preceq
    DZ_\pi
    \preceq
    Z_\pi.
\]
By induction, it follows that
\[
    \mathbf{0}
    \preceq
    (DZ_\pi)^n
    \preceq
    Z_\pi^n,
    \qquad
    n\geq1.
\]
Indeed, if $(DZ_\pi)^n\preceq Z_\pi^n$, then
\[
    (DZ_\pi)^{n+1}
    =
    (DZ_\pi)^nDZ_\pi
    \preceq
    Z_\pi^nDZ_\pi
    \preceq
    Z_\pi^{n+1}.
\]

By Lemma~\ref{lem:uniform-transience}, the series
\[
    \sum_{n=1}^{\infty}Z_\pi^n
\]
converges. The preceding componentwise bound therefore implies that
\[
    \sum_{n=1}^{\infty}(DZ_\pi)^n
\]
also converges. Hence $\rho(DZ_\pi)<1$, so the Neumann-series identity
gives
\[
    (I-DZ_\pi)^{-1}
    =
    \sum_{n=0}^{\infty}(DZ_\pi)^n.
\]
Multiplying by $DZ_\pi$ yields
\[
    (I-DZ_\pi)^{-1}DZ_\pi
    =
    \sum_{n=1}^{\infty}(DZ_\pi)^n.
\]
Moreover,
\[
    \mathbf{0}
    \preceq
    \sum_{n=1}^{\infty}(DZ_\pi)^n
    \preceq
    \sum_{n=1}^{\infty}Z_\pi^n.
\]
Since the matrices are nonnegative, the infinity norm preserves this
ordering. Consequently,
\[
    \left\|
        (I-DZ_\pi)^{-1}DZ_\pi
    \right\|_\infty
    \leq
    \left\|
        \sum_{n=1}^{\infty}Z_\pi^n
    \right\|_\infty
    \leq
    K,
\]
where the final inequality follows from
Lemma~\ref{lem:uniform-transience}.
\end{proof}

We now obtain a uniform bound for each fixed policy.

\begin{lemma}
\label{lem:fixed-policy-absorbing-gap}
Under Assumption~\ref{ass:all-policies-proper}, for every
$\pi\in\Stat$ and $\gamma\in(0,1)$,
\[
    \mathbf{0}
    \preceq
    \bar{x}_{\pi,\gamma}
    -
    X_{\pi,\gamma}
    \preceq
    K(1-\gamma)\ones,
\]
where $K$ is the constant from
Lemma~\ref{lem:uniform-transience}.
\end{lemma}

\begin{proof}
Fix $\pi\in\Stat$. Throughout the proof, all vectors and matrices are
restricted to the set $\mathcal Z$ of nonabsorbing state-action pairs.
Let
\[
    x
    :=
    \bar{x}_{\pi,\gamma},
    \qquad
    v
    :=
    X_{\pi,1},
\]
and define
\[
    q_\pi
    :=
    \ones-Z_\pi\ones.
\]
For each $(s,a)\in\mathcal Z$, the component $q_\pi(s,a)$ is the
probability that the process enters the absorbing set after one
transition from $(s,a)$.

Recall that
\[
    D
    :=
    \operatorname{diag}
    \left(
        e^{-\widehat{\theta}r(s,a)}
    \right)_{(s,a)\in\mathcal Z}.
\]
Since the value at every absorbing state-action pair equals one, the
fixed-point equation for $\bar{x}_{\pi,\gamma}$, restricted to
$\mathcal Z$, can be written as
\[
    x
    =
    D\left(
        q_\pi+Z_\pi x^\gamma
    \right),
\]
where the power is applied componentwise. Similarly, the undiscounted
absorbing criterion satisfies
\[
    v
    =
    D\left(
        q_\pi+Z_\pi v
    \right).
\]
Subtracting these equations gives
\begin{align*}
    x-v
    &=
    DZ_\pi(x^\gamma-v)
    \\
    &=
    DZ_\pi(x^\gamma-x)
    +
    DZ_\pi(x-v).
\end{align*}
Therefore,
\[
    x-v
    =
    (I-DZ_\pi)^{-1}DZ_\pi
    (x^\gamma-x).
\]

By Lemma~\ref{lem:absorbing-order},
\[
    \mathbf{0}
    \prec
    x
    \preceq
    \ones.
\]
For every scalar $y\in[0,1]$, we have
\[
    y\leq y^\gamma.
\]
Moreover, since $y\mapsto y^\gamma$ is concave, its tangent at
$y=1$ gives
\[
    y^\gamma
    \leq
    \gamma y+(1-\gamma).
\]
Consequently,
\[
    0
    \leq
    y^\gamma-y
    \leq
    (1-\gamma)(1-y)
    \leq
    1-\gamma.
\]
Applying this componentwise yields
\[
    \mathbf{0}
    \preceq
    x^\gamma-x
    \preceq
    (1-\gamma)\ones.
\]

Let
\[
    R_\pi
    :=
    (I-DZ_\pi)^{-1}DZ_\pi.
\]
By Lemma~\ref{lem:absorbing-resolvent}, $R_\pi$ is entrywise
nonnegative and
\[
    \|R_\pi\|_\infty
    \leq
    K.
\]
It follows that
\begin{align*}
    \mathbf{0}
    \preceq
    x-v
    &=
    R_\pi(x^\gamma-x)
    \\
    &\preceq
    (1-\gamma)R_\pi\ones
    \\
    &\preceq
    (1-\gamma)\|R_\pi\|_\infty\ones
    \\
    &\preceq
    K(1-\gamma)\ones.
\end{align*}

Finally, Lemma~\ref{lem:absorbing-order} gives
\[
    X_{\pi,1}
    \preceq
    X_{\pi,\gamma}
    \preceq
    \bar{x}_{\pi,\gamma}.
\]
Consequently,
\[
    \mathbf{0}
    \preceq
    \bar{x}_{\pi,\gamma}
    -
    X_{\pi,\gamma}
    \preceq
    \bar{x}_{\pi,\gamma}
    -
    X_{\pi,1}
    \preceq
    K(1-\gamma)\ones.
\]
At state-action pairs belonging to the absorbing set, all three
quantities equal one, so the same conclusion holds on the entire
state-action space.
\end{proof}

We can now compare the optimal values of the two criteria. Since a
single stationary policy need not simultaneously attain the minimum of
$X_{\pi,\gamma}(s,a)$ at every state-action pair, we define the
componentwise optimal-value envelopes
\[
    X_\gamma^\star(s,a)
    :=
    \inf_{\pi\in\Stat}
    X_{\pi,\gamma}(s,a)
\]
and
\[
    \bar{x}_\gamma^\star(s,a)
    :=
    \inf_{\pi\in\Stat}
    \bar{x}_{\pi,\gamma}(s,a).
\]

\begin{theorem}[Asymptotic agreement of the optimal values]
\label{thm:absorbing-optimal-agreement}
Under Assumption~\ref{ass:all-policies-proper}, for every
$\gamma\in(0,1)$,
\[
    \mathbf{0}
    \preceq
    \bar{x}_\gamma^\star-X_\gamma^\star
    \preceq
    K(1-\gamma)\ones,
\]
where $K$ is the constant from
Lemma~\ref{lem:uniform-transience}. Consequently,
\[
    \left\|
        \bar{x}_\gamma^\star-X_\gamma^\star
    \right\|_\infty
    \leq
    K(1-\gamma)
    \longrightarrow
    0
    \qquad
    \text{as }\gamma\uparrow1.
\]
\end{theorem}

\begin{proof}
Fix an admissible state-action pair $(s,a)$. By
Lemma~\ref{lem:fixed-policy-absorbing-gap}, for every $\pi\in\Stat$,
\[
    X_{\pi,\gamma}(s,a)
    \leq
    \bar{x}_{\pi,\gamma}(s,a)
    \leq
    X_{\pi,\gamma}(s,a)
    +
    K(1-\gamma).
\]
Since the sets
\[
    \left\{
        X_{\pi,\gamma}(s,a):
        \pi\in\Stat
    \right\}
    \qquad\text{and}\qquad
    \left\{
        \bar{x}_{\pi,\gamma}(s,a):
        \pi\in\Stat
    \right\}
\]
are nonempty and bounded below by zero, their infima are well defined.

Taking the infimum over $\pi\in\Stat$ in the first inequality gives
\[
    X_\gamma^\star(s,a)
    \leq
    \bar{x}_\gamma^\star(s,a).
\]
Similarly, taking the infimum over $\pi\in\Stat$ in the second
inequality gives
\begin{align*}
    \bar{x}_\gamma^\star(s,a)
    &\leq
    \inf_{\pi\in\Stat}
    \left[
        X_{\pi,\gamma}(s,a)
        +
        K(1-\gamma)
    \right]
    \\
    &=
    \inf_{\pi\in\Stat}
    X_{\pi,\gamma}(s,a)
    +
    K(1-\gamma)
    \\
    &=
    X_\gamma^\star(s,a)
    +
    K(1-\gamma),
\end{align*}
where the second equality follows because $K(1-\gamma)$ does not
depend on $\pi$. Thus,
\[
    0
    \leq
    \bar{x}_\gamma^\star(s,a)
    -
    X_\gamma^\star(s,a)
    \leq
    K(1-\gamma).
\]
Since $(s,a)$ was arbitrary, we conclude that
\[
    \mathbf{0}
    \preceq
    \bar{x}_\gamma^\star
    -
    X_\gamma^\star
    \preceq
    K(1-\gamma)\ones.
\]
Taking the infinity norm yields
\[
    \left\|
        \bar{x}_\gamma^\star
        -
        X_\gamma^\star
    \right\|_\infty
    \leq
    K(1-\gamma).
\]
Since the right-hand side converges to zero as $\gamma\uparrow1$, the
desired result follows.
\end{proof}

\section{Analysis for the Two-Timescale Fixed Point Method}
Here we discuss the stability, asymptotic convergence, and convergence rates of the two-timescale fixed-point method in \eqref{e:2TS.algorithm.update}. 

\subsection{Stability}
We first prove the almost sure boundedness of the sequences  $(Q_n)$ and $(g_n)$.   

\begin{lemma}
\label{lem:2TS.stability}
Suppose $Q_0 \in \bR^{SA}$ and $g_0 \in \bRp^{SA}$ be such that
$\max\{\|Q_0\|_\infty, \frac{\gamma}{\theta} \|\ln g_0\|_\infty \} \leq \Cm$, where $\Cm := \frac{\|r\|_\infty}{1-\gamma}$. Also, suppose $\alpha_n, \beta_n \in [0, 1]$ for all $n \geq 0.$ Then, for each $n \geq 0,$ we have
\begin{equation}
    \max\left\{\|Q_n\|_\infty, \frac{\gamma}{\theta} \|\ln g_n\|_\infty \right\} \leq   \Cm.
\end{equation}
\end{lemma}
\begin{proof} 
We prove using induction. For $n=0,$ the claim holds trivially. Now assume that 
\[
    \|Q_n\|_\infty \le \Cm
    \qquad\text{and}\qquad
    \frac{\gamma}{\theta}\|\ln g_n\|_\infty \le \Cm
\]
for some $n\ge 0.$ The second inequality implies, for all $(s,a)$,
\begin{equation}
\label{e:gn.Bd}
    \exp\left[-\frac{\theta}{\gamma}\Cm\right] \leq g_n(s,a)\leq \exp\left[\frac{\theta}{\gamma}\Cm\right].
\end{equation}
The rest of the proof proceeds via three steps.
\paragraph{Step 1: Bound $\dfrac{\gamma}{\theta}\|\ln g_{n+1}\|_\infty$.}
For $(s,a)\neq (s_n, a_n)$, we have $g_{n+1}(s,a)=g_n(s,a);$ hence, for every $(s,a)\neq (s_n, a_n)$, we have
\[
    \frac{\gamma}{\theta}|\ln g_{n+1}(s,a)| =\frac{\gamma}{\theta}|\ln g_n(s,a)| \le \Cm.
\]
For the updated coordinate, we have
\begin{equation}
\label{e:g_n.conv.comb}
    g_{n+1}(s_n, a_n)=(1-\beta_n)g_n(s_n,a_n)+\beta_n\hat G_n.
\end{equation} 
Also, $\hat G_n = \exp\!\left[-\frac{\theta}{\gamma}\Bigl(r(s_n,a_n)+\gamma\max_{\ap}Q_n(s_{n+1},\ap)\Bigr)\right],$  and
\[
    \bigl|r(s_n,a_n)+\gamma\max_{\ap}Q_n(s_{n+1},\ap)\bigr| \le \|r\|_\infty+\gamma \Cm \le \Cm,
\]
where the first inequality follows since $|r(s_n, a_n)| \leq \|r\|_\infty$ and $\|Q_n\|_\infty \leq \Cm,$ while the last inequality uses $\Cm\ge \|r\|_\infty/(1-\gamma)$.
Therefore,
\begin{equation}
\label{e:h.G.Bd}
    \exp\left[-\frac{\theta}{\gamma}\Cm\right] \leq \hat G_n \leq \exp\left[\frac{\theta}{\gamma}\Cm\right].
\end{equation}
Since $\beta_n\in[0,1],$ it follows from \eqref{e:gn.Bd}, \eqref{e:h.G.Bd}, and \eqref{e:g_n.conv.comb} that 
\[
    \exp\left[-\frac{\theta}{\gamma}\Cm\right] \leq g_{n+1}(s_n,a_n)\leq \exp\left[\frac{\theta}{\gamma}\Cm\right].
\]
Hence, 
\begin{equation}
\label{e:g.n+1.Bd}
    \frac{\gamma}{\theta}\|\ln g_{n+1}\|_\infty\le \Cm.
\end{equation} 

\paragraph{Step 2: Bound $\|Q_{n+1}\|_\infty$.} For all $(s, a),$ our hypothesis implies that 
\begin{align*}
    |Q_{n + 1}(s, a)| \leq {} & (1 - \alpha_n) |Q_n(s, a)| +  \alpha_n \frac{\gamma}{\theta} |\ln g_n(s, a)| \\
    \leq {} & (1 - \alpha_n) \Cm + \alpha_n \Cm \\
    = {} & \Cm.
\end{align*}

Combining Steps 1 and 2 proves
\[
\max\left\{\|Q_{n+1}\|_\infty,\ \frac{\gamma}{\theta}\|\ln g_{n+1}\|_\infty\right\}\le \Cm,
\]
closing the induction. The desired result now follows.
\end{proof}

\subsection{Asymptotic Convergence}

Here we establish $(Q_n)$ and $(g_n)$'s asymptotic convergence. Let $(\cF_n)_{n \geq 0}$ be the filtration given by 
\begin{equation}
    \label{e:2TS.Filtration}
    \cF_n = \sigma(Q_0, g_0, s_0, a_0, \ldots, s_n, a_n).
\end{equation}
\begin{proof}[\textbf{Proof of Theorem \ref{thm:twoTSconv}}].
We first analyze $(g_n)$'s asymptotic behavior. For that, we consider the update rules in \eqref{e:2TS.algorithm.update}  from the $(\beta_n)$-stepsize perspective. That is, we rewrite \eqref{e:2TS.algorithm.update} as 
\[
    z_{n + 1} = z_n + \beta_n [\phi(z_n, s_n, a_n, s_{n + 1}) + \epsilon_n ], \quad n \geq 0,
\]
where $ z_n = \begin{bmatrix}
        Q_{n} \\ g_{n}
    \end{bmatrix},$ $\phi: \bR^{2SA} \times \cS \times \cA \times \cS \to \bR^{2SA}$ is given by
\begin{align}
    \phi(z, s, a, s') = {} &
    \begin{bmatrix}
        0 \\
        e_{s, a} [G(Q, s, a, s') - g(s, a)]
    \end{bmatrix} \nonumber \\
    = {} &
    \begin{bmatrix}
        0 \\
        e_{s, a} \bigg[\exp\left(-\frac{\theta}{\gamma} r(s, a) - \theta \max_{\ap} Q(\sp, \ap) \right) - g(s, a)\bigg]
    \end{bmatrix} 
    \label{e:phi.defn}
\end{align}
for any $ z = 
    \begin{bmatrix}
        Q \\ g
    \end{bmatrix},$ and
\[
    \epsilon_n = 
    \begin{bmatrix}
        \frac{\alpha_n}{\beta_n} e_{s_n, a_n} [-\frac{\gamma}{\theta} \ln g_n(s_n, a_n) - Q_n(s_n, a_n) ] \\
        0
    \end{bmatrix},
\]
which again is a vector in $\bR^{2SA}.$ Clearly, $z_n$ is $\cF_n$-measurable for each $n \geq 0.$ With this in mind, the above recursion can be equivalently expressed in the standard stochastic approximation form
\[
    z_{n + 1} = z_n + \beta_n[\bar{\phi}(z_n, s_n, a_n) + \epsilon_n + M_{n + 1}], \quad n \geq 0.
\]
Here, $\bar{\phi}(z_n, s_n, a_n) := \bE [\phi(z_n, s_n, a_n, s_{n + 1})\mid \cF_n] = \bE [\phi(z_n, s_n, a_n, s_{n + 1})\mid z_n, s_n, a_n].$ More explicitly, for $z \in \bR^{2SA}$ and $(s, a) \in \cS \times \cA,$ we have
\begin{align}
    \bar{\phi}(z, s, a) =
        {} & 
        \bE [\phi (z_n, s_n, a_n, s_{n + 1})\mid z_n = z, s_n = s, a_n = a] \nonumber \\  
        = {} & 
        \begin{bmatrix}
            0 \\
            e_{s, a} \bigg[\sum_{s' \in \cS} \cP(\sp|s, a) \exp\left(-\frac{\theta}{\gamma} r(s, a) - \theta \max_{\ap} Q(\sp, \ap) \right) - g(s, a)\bigg]
        \end{bmatrix} \nonumber \\
        = {} & 
        \begin{bmatrix}
            0
            \\
            e_{s, a} \left(\exp\!\left[-\tfrac{\theta}{\gamma}\,TQ (s,a)\right] - g(s, a) \right)
        \end{bmatrix}, \label{e:bar.phi.defn}
\end{align}
where the last relation follows from \eqref{e:T.defn}.
The noise term 
\begin{equation}
    \label{e:M.n+1.defn}
    M_{n + 1} = \phi(z_n, s_n, a_n, s_{n + 1}) - \bar{\phi}(z_n, s_n, a_n),
\end{equation} 
and it satisfies $\bE[M_{n + 1}|\cF_n] = 0$ \as Hence,  $(M_n)_{n \geq 1}$ forms a Martingale-difference sequence with respect to the filtration $(\cF_n).$ Also, since $z_n \in \cK$ from Lemma~\ref{lem:2TS.stability}, we have from \eqref{e:phi.defn} and \eqref{e:bar.phi.defn} that 
\[
    \|M_{n + 1}\|_\infty \leq  2 \exp\left[\frac{\theta}{\gamma} \Cm\right] \as
\]
This relation also shows that  $(M_n)$ is square-integrable and satisfies
\[
    \mathbb{E}\!\left[\|M_{n+1}\|_\infty^{2}\mid \mathcal{F}_{n}\right] \leq 4 \exp\left[\frac{2\theta}{\gamma} \Cm\right] \leq 4 \exp\left[\frac{2\theta}{\gamma} \Cm\right]  \bigl(1+\|z_n\|_\infty^{2}\bigr)
    \quad \text{a.s.}
\]
Finally, since $z_n \in \cK$ from Lemma~\ref{lem:2TS.stability} for each $n \geq 0$ and since $\alpha_n/\beta_n \to 0$,
we have $\epsilon_n \xrightarrow{\mathrm{a.s.}} 0$. 

Now define the averaged drift $\tilde{\phi} : \bR^{2SA} \to \bR^{2SA}$ by
\[
\tilde{\phi}(z)
:= \sum_{(s,a)\in\cS\times\cA} \pi_Q(s,a)\,\bar{\phi}(z,s,a)
=
\begin{bmatrix}
0 \\
D_{\pi_Q}\!\left[\exp\!\left(-\tfrac{\theta}{\gamma}\,TQ\right) - g\right]
\end{bmatrix},
\]
where $\pi_Q,$ recall, denotes the stationary distribution of the Markov chain with
transition kernel $K_Q$, and $D_{\pi_Q} \in \bR^{SA \times SA}$ is the diagonal
matrix whose $(s,a)$-th diagonal entry equals $\pi_Q(s,a)$. It then follows from \cite[Corollary~8.1 and Remark~3, Section~2.2]{borkar2008stochastic}
that the sequence $(z_n)$ converges almost surely to a connected, internally
chain transitive invariant set of the ODE
\[
    \dot{z}(t) = \tilde{\phi}\bigl(z(t)\bigr).
\]
The above ODE admits
\[
\mathcal A :=
\left\{
\begin{bmatrix}
Q \\[1mm]
\exp\!\left[-\tfrac{\theta}{\gamma}\,TQ\right]
\end{bmatrix}
:\;
Q \in \bR^{SA}
\right\}
\]
as its unique global attractor. Indeed, since $\dot Q(t) = 0$, the $Q$-component
remains constant along trajectories. For a fixed $Q$, the dynamics of the
$g$-component reduces to
\[
\dot g(t)
= D_{\pi_Q}
\Bigl(
\exp\!\left[-\tfrac{\theta}{\gamma}\,TQ\right] - g(t)
\Bigr),
\]
which is a linear ODE with
$\exp\!\left[-\tfrac{\theta}{\gamma}\,TQ\right]$ as its globally asymptotically
stable equilibrium. Therefore, we conclude that $z_n \xrightarrow{\mathrm{a.s.}} \mathcal A$, or
equivalently,
\begin{equation}
\label{e:gn.asym.beh}
\left\| g_n - \exp\!\left[-\tfrac{\theta}{\gamma}\,TQ_n \right] \right\|
\xrightarrow{\mathrm{a.s.}} 0,
\end{equation}
as claimed.

We now discuss $(Q_n)$'s behavior from the $(\alpha_n)$-stepsize perspective. The update rule for $Q_n$ can be rewritten as 
\[
    Q_{n + 1} = Q_n + \alpha_n [T Q_n - Q_n + \epsp_n],
\]
where 
\[
    \epsp_{n} = \left[- \frac{\gamma}{\theta} \ln g_n - TQ_n\right].
\]
Now, from \eqref{e:gn.asym.beh}, we have $\epsp_n \overset{\as}{\longrightarrow} 0.$ Hence, from \cite[Corollary~8.1 and Remark 3 of Section 2.2 ]{borkar2008stochastic}, we have that $(Q_n)$ converges to a connected, internally chain transitive invariant set of the ODE 
\begin{equation}
\label{e:Q.lim.ODE}
    \dot{Q}(t) = T Q(t) - Q(t).
\end{equation}
To complete the proof, it remains to show that $\Qstr$ is the unique globally asymptotically stable equilibrium of the above ODE. 




%
For $t \geq 0,$ define $z(t) := Q(t) - \Qstr$ and $u(t) := \|z(t)\|_\infty.$ Since $T \Qstr = \Qstr,$ we have
\begin{equation}
\label{e:z.i.derivative}
    \dot{z}_i(t) = (T Q(t))_i - (T \Qstr)_i - z_i(t), \quad i \in \cS \times \cA.
\end{equation}
Since $Q(\cdot)$ is absolutely continuous and $\|\cdot\|_\infty$ is Lipchitz continuous, it follows that $u(\cdot)$ is absolutely continuous and hence  differential for a.e. $t \geq 0.$ Fix such a $t,$ and let $i_t$ be any index satisfying
\[
    |z_{i_t}(t)| = \|z(t)\|_\infty = u(t).
\]
Then, 
\[
    \frac{\textnormal{d}}{\textnormal{d}t} u(t) = \sign(z_{i_t}(t)) \dot{z}_{i_t}(t).
\]
Substituting \eqref{e:z.i.derivative} with $i = i_t,$ and using the $\gamma$-contraction property of $T$, we obtain
\begin{equation}
     \label{e:T.contraction.impact}
     \frac{\textnormal{d}}{\textnormal{d}t} u(t) = \sign(z_{i_t}(t)) \bigg((T Q(t))_{i_t} - (T \Qstr)_{i_t} \bigg) - u(t) \leq - (1 - \gamma) u(t).
\end{equation}
Since this is true for a.e. $t \geq 0,$ applying the differential form of Gronwall's inequality yields
\[
    \|Q(t) - \Qstr\|_\infty \leq e^{-(1 - \gamma)t} \|Q(0) - \Qstr\|_\infty, \quad t \geq 0.
\]
Thus, every solution trajectory of \eqref{e:Q.lim.ODE} converges exponentially fast to $\Qstr.$ In particular, $\Qstr$ is the unique globally asymptotically stable equilibrium of \eqref{e:Q.lim.ODE}, as desired.
\end{proof}

\subsection{Convergence Rates}

Throughout this section, we assume that at each iteration $n\geq0,$ the state-action pair $(s_n,a_n)$ is sampled from a stationary distribution $\nu,$ while the next state $s'_n$ is sampled from $\cP(\cdot|s_n,a_n).$ Also let $D^\nu:= \diag(\nu(s,a): (s,a)\in \cS\times\cA)$ and $\lambda:= \min_{s,a} \nu(s,a) > 0.$



Before we proceed, we define some relevant notation. Define for $0\leq n_1<n_2,$
\[
    \Gbeta_{n_1:n_2} := \prod_{n=n_1}^{n_2-1}\left( 1-2\lambda\beta_n + 3\beta^2_n \right) \quad \text{and} \quad \Galpha_{n_1:n_2} := \prod_{n=n_1}^{n_2-1}\left( 1-(1-\gamma)\alpha_n \right).
\]
Also, define the constants
\[
    C_g := e^{\frac{2\theta}{\gamma}\Cm}\bigg[ 4e^{\frac{6\beta}{2\beta-1}} e^{\frac{2}{1-\beta}}\left(\frac{\beta}{2e}\right)^{\frac{\beta}{1-\beta}} + 2C'(1+6\theta^2\Cm^2)\bigg]^{1/2},
\]
\[
    C^{(1)}_Q := 2\Cm e^{\frac{(1-\gamma)}{(1-\alpha)}}, \quad C^{(2)}_Q := \frac{\gamma}{\theta}C_g C'' e^{\frac{\theta}{\gamma}\Cm},
\]
where
   \begin{equation}
        \begin{aligned}
            C' &:= e^{\frac{4\beta}{2\beta-1}}\cdot \max\left\{ 2^{-\beta}, \left[ \frac{4\beta}{2^\beta e(1-\beta)}\right]^{\frac{\beta}{1-\beta}}\right\},
            \\
            C'' &:= \frac{2}{1-\gamma}\cdot\max\left\{2^{-\frac{\beta}{2}}, \left[ \frac{2\beta}{2^\alpha e(1-\gamma)(1-\alpha)}\right]^{\frac{\beta}{2(1-\alpha)}} \right\}.
        \end{aligned}
    \end{equation}
Finally, we present the following technical lemma, which we prove later.
\begin{lemma}\label{lem: bounds.contractive.prod}
    For every $0\leq k\leq N,$ we have
    \begin{enumerate}
        \item $\Gbeta_{k:N} \leq \CGbeta \exp\bigg(-2\lambda\sum_{n=k}^{N-1}\beta_n\bigg) \quad \text{and} \quad  
        \Galpha_{k:N} \leq \exp\bigg(-(1-\gamma)\sum_{n=k}^{N-1}\alpha_n \bigg);$
        \item $\sum_{n=k}^{N-1}\beta_nf_n\ \Gbeta_{n+1:N} \leq \CGbeta^{a,b}\ f_N \quad \text{and} \quad \sum_{n=k}^{N-1}\alpha_n f_n \ \Galpha_{n+1:N}\leq \CGalpha^{a,b}\ f_N,$
    \end{enumerate}
    %
    %
    where $f_n:= \alpha^a_n\beta^b_n,$ for some $a,b\in \bZ\cup\{1/2\},$  and the constants are defined as
    \[
        \CGbeta := \exp\left(\frac{6\beta}{2\beta-1}\right), \quad \CGbeta^{a,b}:= \CGbeta\cdot \max\left\{ 2^{-(a\alpha+b\beta)}, \left[ \frac{a\alpha+b\beta}{\lambda e(1-\beta)2^\beta}\right]^{\frac{a\alpha+b\beta}{1-\beta}}\right\},
    \]
    and
    \[
        \CGalpha^{a,b}:= \frac{2}{1-\gamma}\max\left\{2^{-(a\alpha+b\beta)}, \left[ \frac{4(a\alpha+b\beta)}{2^\alpha e(1-\gamma)(1-\alpha)}\right]^{\left(\frac{a\alpha+b\beta}{1-\alpha}\right)} \right\}.
    \]
\end{lemma}

\begin{proof}[Proof of Theorem~\ref{thm:2TS.Rates}]
We first discuss $(g_n)$'s behavior from the $(\beta_n)$-stepsize perspective. Let $y_n = g_n - e^{-\theta/\gamma TQ_n}.$ Then, we have for all $n\geq  0,$
\begin{align}\label{e: two.timescale.fast}
    y_{n + 1} = {} & g_{n + 1} - e^{-\frac{\theta}{\gamma} TQ_{n + 1}} 
    \nonumber\\
    \overset{(a)}{=} {} & y_n +  [e^{-\frac{\theta}{\gamma} TQ_n} - e^{-\frac{\theta}{\gamma} TQ_{n + 1}}] + \beta_n e_{s_n, a_n}\left[\hG_n - g_n(s_n, a_n)\right] 
    \nonumber\\
    \overset{(b)}{=} {} & y_n + \epsilon_n + \beta_n e_{s_n, a_n}\left[e^{-\frac{\theta}{\gamma}TQ_n(a_n, s_n)} - g_n(s_n, a_n) + \hG_n - e^{-\frac{\theta}{\gamma}TQ_n(a_n, s_n)}\right] 
    \nonumber\\
    \overset{(c)}{=} {} & y_n - \beta_n e_{s_n, a_n} y_n (s_n, a_n) + \beta_n e_{s_n, a_n}\left[\hG_n - e^{-\frac{\theta}{\gamma}TQ_n(a_n, s_n)} \right] + \epsilon_n,
\end{align}
where $(a)$ follows from the update rule of $(g_n),$ $(b)$ follows from defining $\epsilon_n:=e^{-(\theta/\gamma) TQ_n} - e^{-(\theta/\gamma) TQ_{n + 1}},$ and $(c)$ uses the definition of $y_n.$

For $n\geq 0,$ define $M_{n+1}:= e_{s_n,a_n}\left[\hG_n - e^{-\frac{\theta}{\gamma}TQ_n(a_n, s_n)} \right]$ and let $\bE_n[\cdot]$ denote conditional expectation w.r.t the $\sigma$-algebra
\begin{equation}
    \cF'_n := \sigma(Q_0, g_0, s_0, a_0,s'_0, \ldots, s_{n-1}, a_{n-1}, s'_{n-1}).
\end{equation}
Then, squaring both sides of \eqref{e: two.timescale.fast} and taking expectation gives us
\begin{align}\label{e: two.ts.recursive}
    &\bE\|y_{n+1}\|^2_2
    \nonumber\\
    &\overset{(a)}{=} \bE\|\left(I-\beta_n e_{s_n,a_n} e^\top_{s_n,a_n} \right)y_n\|_2^2 +\bE\|\epsilon_n\|_2^2
    + \beta_n^2\bE\|M_{n+1}\|_2^2
    \nonumber\\
    & \quad + 2\beta_n \bE\left\langle \bE_n[\epsilon_n], \bE\left[I-\beta_n e_{s_n,a_n} e^\top_{s_n,a_n} \right]y_n \right\rangle 
    \nonumber\\
    & \quad + 2\beta_n\bE\left\langle \bE_n[M_{n+1}], \bE\left[I-\beta_n e_{s_n,a_n} e^\top_{s_n,a_n} \right]y_n \right\rangle + 2\beta_n\bE\left[\left\langle M_{n+1}, \epsilon_n \right\rangle\right]
    \nonumber\\
    & \overset{(b)}{\leq} \bE\|\left(I-\beta_n e_{s_n,a_n} e^\top_{s_n,a_n} \right)y_n\|_2^2 + \bE\|\epsilon_n\|_2^2 + \beta^2_n\bE\|M_{n+1}\|_2^2 
    \nonumber\\
    & \quad + \beta_n\left(\beta_n\bE\|(I-\beta_nD^\nu)y_n\|_2^2 + \frac{1}{\beta_n}\bE\|\bE_n[\epsilon_n]\|_2^2\right)
    + \beta_n\left(\beta_n\bE\|M_{n+1}\|_2^2 + \frac{1}{\beta_n}\bE\|\epsilon_n\|_2^2 \right)
    \nonumber\\
    & \overset{(c)}{\leq} \bE\|\left(I-\beta_n e_{s_n,a_n} e^\top_{s_n,a_n} \right)y_n\|_2^2  +  2\beta^2_n\bE\|M_{n+1}\|_2^2 + \beta^2_n\bE\|(I-\beta_nD^\nu)y_n\|_2^2 + 3\bE\|\epsilon_n\|_2^2
    \nonumber\\
    & \overset{(d)}{\leq} \bE\|\left(I-\beta_n e_{s_n,a_n} e^\top_{s_n,a_n} \right)y_n\|_2^2  +  2e^{2\theta\Cm}\beta^2_n + \beta^2_n\bE\|(I-\beta_nD^\nu)y_n\|_2^2 + 3\bE\|\epsilon_n\|_2^2,
\end{align}
where $(a)$ uses the fact that $e_{s_n,a_n}$ is independent of $\cF'_n$ while $y_n$ is $\cF'_n$-measurable, whereas $(b)$ follows by applying the Young's inequality, taking $(s_n,a_n)\sim \nu,$ and using the fact that $\bE_n[M_{n+1}]=0$ since
\[
    \bE_n[\hG_n] = \bE_n[G(Q_n, s_n,a_n,s_{n+1})] = \sum_{s'} P(s'|s_n,a_n)G(Q_n,s_n,a_n, s') = e^{-\frac{\theta}{\gamma}TQ_n}.
\]
Further, $(c)$ uses Jensen's inequality for conditional expectation, whereas $(d)$ follows since $r(s,a)\geq 0$ and $Q_n(s,a) \geq -\Cm,$ giving
\begin{align}
    \|M_{n+1}\|_2\leq 2|\hG_n| \leq \exp\left(-\frac{\theta}{\gamma}r(s_n, a_n) -\theta \max_{a'} Q_n(s_{n+1},a')\right) \leq e^{ \theta\Cm }.
\end{align}
Finally, note that 
\begin{align}
    |\epsilon_n(s,a)| & \overset{(a)}{\leq} \sum_{s'}P(s'|s,a)\left| e^{-\theta\max_{a'}Q_n(s',a')} - e^{-\theta\max_{a'}Q_{n+1}(s',a')} \right| 
    \nonumber\\
    & \overset{(b)}{\leq} \theta e^{\theta\Cm}\sum_{s'}P(s'|s,a)\left| \max_{a'}Q_n(s',a') - \max_{a'}Q_{n+1}(s',a')\right|
    \nonumber\\
    & \leq  \theta e^{\theta\Cm} \|Q_n - Q_{n+1}\|_\infty
    \nonumber \\
    & \overset{(c)}{\leq} \theta e^{\theta\Cm}\left(\frac{\gamma}{\theta}\|\ln g_n\|_\infty +  \|Q_n\|_\infty \right)\alpha_n \overset{(d)}{\leq} 2\Cm\theta e^{\theta\Cm}\alpha_n,
\end{align}
where $(a)$ uses the definition of $T,$ $(b)$ follows from Lemma~\ref{lem:2TS.stability} since $u\mapsto\exp(-\theta u)$ is $\theta\exp(\theta\Cm)$-Lipschitz on $[-\Cm,\Cm],$ $(c)$ uses the update rule for $(Q_n),$ and $(d)$ uses Lemma~\ref{lem:2TS.stability}.

Now, since $y_n$ is $\cF'_n$-measurable, we have
\begin{align*}
    \bE_n\|\left(I-\beta_n e_{s_n,a_n} e^\top_{s_n,a_n} \right)y_n\|_2^2 
    & = y^\top_n \bE_n\left[\left(I-\beta_n e_{s_n,a_n} e^\top_{s_n,a_n} \right)^\top\left(I-\beta_n e_{s_n,a_n} e^\top_{s_n,a_n} \right)\right]y_n 
    \\
    & \overset{(a)}{=} y^\top_n \bE\left[\left(I-\beta_n e_{s_n,a_n} e^\top_{s_n,a_n} \right)^\top\left(I-\beta_n e_{s_n,a_n} e^\top_{s_n,a_n} \right)\right]y_n 
    \\
    & = y^\top_n \bE\left[\left(I-\beta_n e_{s_n,a_n} e^\top_{s_n,a_n} \right)\left(I-\beta_n e_{s_n,a_n} e^\top_{s_n,a_n} \right)\right]y_n
    \\
    & \overset{(b)}{=} y^\top_n \bE\left[ I - 2\beta_n e_{s_n,a_n}e_{s_n,a_n}^\top + \beta_n^2 e_{s_n,a_n}e_{s_n,a_n}^\top\right]y_n 
    \\
    & \overset{(c)}{=} y^\top_n \bE\left[ I - 2\beta_n D^\nu + \beta_n^2 D^\nu\right]y_n 
    \\
    & \overset{(d)}{\leq} \left( 1 - 2\lambda\beta_n + \beta^2_n \right)\|y_n\|^2_2,
\end{align*}
where $(a)$ follows since $(s_n,a_n)\sim \nu\perp\cF'_n,$ $(b)$ follows since $e_{s_n,a_n}^\top e_{s_n,a_n}=1,$ $(c)$ follows from the fact that $(s_n,a_n)\sim \nu,$ and $(d)$ follows from the definition of $\lambda,$ and using $y^\top D^\nu y \leq y^\top y.$

Furthermore, $(d)$ implies
\[
    \|(I-\beta_nD^\nu)y_n\|^2_2 \leq \left( 1 - 2\lambda\beta_n + \beta_n^2 \right)\|y_n\|^2_2.
\]

Plugging the above bounds into \eqref{e: two.ts.recursive} gives us
\begin{align}\label{e: two.ts.fast.final}
    \bE\|y_{n+1}\|^2_2
    & \leq (1+\beta^2_n)\left( 1 - 2\lambda\beta_n + \beta_n^2 \right)\bE\|y_n\|_2^2 +  2e^{2\theta\Cm}\beta_n  + 12\theta^2\Cm^2e^{2\theta\Cm}\alpha^2_n,
    \nonumber\\
    & \overset{(a)}{\leq} (1-2\lambda\beta_n + 3\beta_n^2)\bE\|y_n\|_2^2 + 2e^{2\theta\Cm}( 1+ 6\theta^2\Cm^2)\beta^2_n,
\end{align}
where $(a)$ follows as we choose $\beta_n$ such that $(1-2\lambda\beta_n+\beta)\leq (1-2\lambda\beta_n+3\beta)\leq 1.$
Then, recursively applying \eqref{e: two.ts.fast.final} gives
\begin{align}\label{e: fast.ts.final.bound}
    \bE\|y_N\|^2 \leq \Gbeta_{0:N}\left\|g_0 - e^{-\frac{\theta}{\gamma}TQ_0}\right\|_2^2 + 2e^{2\theta\Cm}( 1+ 6\theta^2\Cm^2)\sum_{n=0}^{N-1}\beta^2_n\Gbeta_{n+1:N}.
\end{align}
Applying Lemma~\ref{lem: bounds.contractive.prod} to \eqref{e: fast.ts.final.bound} gives us
\begin{align}
    &\bE\|y_N\|_2^2 
    \nonumber\\
    & \overset{(a)}{\leq} \CGbeta\left\|g_0 - e^{-\frac{\theta}{\gamma}TQ_0}\right\|_2^2 e^{-2\lambda\sum_{n=0}^{N-1}\beta_n} 
    + 2\CGbeta^{0,1}e^{2\theta\Cm}( 1+ 6\theta^2\Cm^2)\beta_N
    \nonumber\\
    & \overset{(b)}{\leq} \CGbeta\left\|g_0 - e^{-\frac{\theta}{\gamma}TQ_0}\right\|_2^2 e^{\frac{2}{1-\beta}}\cdot e^{-\frac{2\lambda}{1-\beta}N^{1-\beta}} + 2\CGbeta^{0,1}e^{2\theta\Cm}( 1+ 6\theta^2\Cm^2)\beta_N 
    \nonumber\\
    & \overset{(c)}{\leq} 4\CGbeta  e^{\left(\frac{2}{1-\beta}+ \frac{2\theta\Cm}{\gamma}\right)} e^{-\frac{2\lambda}{1-\beta}(N+1)^{1-\beta}} +  2\CGbeta^{0,1}e^{2\theta\Cm}( 1+ 6\theta^2\Cm^2) \beta_N, 
\end{align}
where in $(a),$ we additionally use the fact that $\alpha^2_n = \beta_n(\alpha^2_n/\beta_n)$ and $\alpha^2_n/\beta_n$ is decreasing in $n.$ Further, $(b)$ follows since $\sum_{n=0}^{N-1}(n+1)^{-\beta} \geq [(N+1)^{1-\beta}-1]/(1-\beta),$ and $(c)$ uses bounds from Lemma~\ref{lem:2TS.stability}.

At last, we use the fact that $x\mapsto x^{\beta}\cdot e^{-\frac{2\lambda}{1-\beta}x^{1-\beta}}$ attains maximum at $x=(\beta/2\lambda)^{1/(1-\beta)}$ to conclude
\begin{align}
    e^{-\frac{2\lambda}{1-\beta}(N+1)^{1-\beta}} \leq \left[(N+1)^\beta\cdot  e^{-\frac{2\lambda}{1-\beta}(N+1)^{1-\beta}}\right]\beta_N  \leq \left(\frac{\beta}{2\lambda e}\right)^{\frac{\beta}{1-\beta}}\beta_N.
\end{align}
Hence, for all $N\geq 0,$
\begin{equation}\label{e: fast.ts.converge.rate}
     \bE\left\|g_N - e^{\frac{-\theta}{\gamma}TQ_N}\right\|_2^2 \leq C^2_g\beta_N,
\end{equation}
where we define 
\[
    C_g := e^{\frac{2\theta}{\gamma}\Cm}\left[ 4\CGbeta e^{\frac{2}{1-\beta}}\left(\frac{\beta}{2\lambda e}\right)^{\frac{\beta}{1-\beta}} + 2\CGbeta^{0,1}(1+6\theta^2\Cm^2)\right]^{1/2}.
\]

This completes the finite-time analysis for $(g_n).$

Now, we discuss $(Q_n)$'s behavior from the $(\alpha_n)$-stepsize perspective. The update rule for $Q_n$ can be rewritten as 
\[
    Q_{n + 1} = Q_n + \beta_n[ TQ_n - Q_n ] + \alpha_n\epsp_n,
\]
where 
\[
    \epsp_n :=  \left[- \frac{\gamma}{\theta} \ln g_n - TQ_n \right].
\]

To proceed, we need to bound $\epsp_n$ appropriately. To do so, note that
\begin{align}\label{e: perturb.slow.ts}
    \|\epsp_n\|_\infty & = \left\|\frac{\gamma}{\theta}\ln g_n + TQ_n\right\|_\infty 
    \nonumber\\
    & \leq \left\|\frac{\gamma}{\theta}\ln g_n + TQ_n\right\|_2
    = \frac{\gamma}{\theta}\left\|\ln g_n - \ln e^{ -\frac{\theta}{\gamma}TQ_n} \right\|_2
    \nonumber\\
    & \overset{(a)}{\leq} \frac{\gamma}{\theta} e^{\frac{\theta\Cm}{\gamma}}\left\| g_n - e^{-\frac{\theta}{\gamma}TQ_n} \right\|_2 = \frac{\gamma}{\theta} e^{\frac{\theta\Cm}{\gamma}}\|y_n\|_2,
\end{align}
where $(a)$ follows since $g_n, e^{-\frac{\theta}{\gamma}TQ_n} \in [\exp(-\theta\Cm/\gamma), \exp(\theta\Cm/\gamma)]^{SA},$ and on any interval $[c,d]$ with $c>0,$ the function $u\mapsto \ln(u)$ is $(1/c)$-Lipschitz.

Then, the update rule for $(Q_n)$ gives us:
\begin{align*}
    \Delta_{n+1} & = \Delta_n + \alpha_n[TQ_n - Q_n] + \alpha_n\epsp_n
    \\
    & \overset{(a)}{=} (1-\alpha_n)\Delta_n + \alpha_n[TQ_n - \Delta_n - Q_n] + \alpha_n\epsp_n
    \\
    & \overset{(b)}{=} (1-\alpha_n)\Delta_n + \alpha_n[TQ_n - \Qstr] + \alpha_n\epsp_n,
\end{align*}
where $(a)$ is obtained by adding and subtracting $\alpha_n\Delta_n,$ and $(b)$ follows since $\Qstr=T\Qstr.$ Now, taking $\|\cdot\|_\infty$ followed by expectation on both sides
\begin{align}
    \bE\|\Delta_{n+1}\|_\infty & \leq (I - \alpha_n)\bE\|\Delta_n\|_\infty + \alpha_n\bE\|TQ_n - T\Qstr)\|_\infty +  \alpha_n\bE\|\epsp_n\|_\infty
    \nonumber \\ 
    & \overset{(a)}{\leq} (1-\alpha_n)\bE\|\Delta_n\|_\infty + \gamma\alpha_n\bE\|\Delta_n\|_\infty +  \alpha_n\bE\|\epsp_n\|_\infty \label{e:T.contraction.impact.rate}
    \\
    & \overset{(b)}{\leq} (1-\alpha_n)\|\Delta_n\|_\infty + \gamma\alpha_n\bE\|\Delta_n\|_\infty  + \frac{\gamma}{\theta} e^{\frac{\theta\Cm}{\gamma}}\alpha_n\bE\|y_n\|_2
    \nonumber  \\
    & = (1-(1-\gamma)\alpha_n)\|\Delta_n\|_\infty   + \frac{\gamma}{\theta} e^{\frac{\theta\Cm}{\gamma}}\alpha_n\sqrt{\bE\|y_n\|^2_2}
    \nonumber \\
    & \overset{(c)}{\leq} (1-(1-\gamma)\alpha_n)\|\Delta_n\|_\infty  + \frac{\gamma}{\theta} e^{\frac{\theta\Cm}{\gamma}}C_g\alpha_n\sqrt{\beta_n}, \nonumber 
\end{align}
where $(a)$ follows since $T$ is a $\gamma$-contraction w.r.t $\|\cdot\|_\infty,$ $(b)$ follows from~\eqref{e: perturb.slow.ts}, and $(c)$ uses \eqref{e: fast.ts.converge.rate}. Iterating the above inequality gives
\begin{align*}
    \bE\|\Delta_N\|_\infty & \leq \Galpha_{0:N}\|\Delta_0\|_\infty + \frac{\gamma}{\theta} e^{\frac{\theta\Cm}{\gamma}}C_g\sum_{n=0}^{N-1} \Galpha_{n+1:N}\alpha_n\sqrt{\beta_n}
    \\
    & \overset{(a)}{\leq} \|\Delta_0\|_\infty e^{-(1-\gamma)\sum_{n=0}^{N-1}\alpha_n} + \frac{\gamma}{\theta} e^{\frac{\theta\Cm}{\gamma}}C_g\CGalpha^{0,1/2}\sqrt{\beta_N}
    \\
    & \overset{(b)}{\leq}   2\Cm e^{\frac{(1-\gamma)}{(1-\alpha)}} e^{-\frac{(1-\gamma)}{(1-\alpha)}(N+1)^{1-\alpha}} + \frac{\gamma}{\theta} e^{\frac{\theta\Cm}{\gamma}}C_g\CGalpha^{0,1/2}\sqrt{\beta_N},
\end{align*}
where $(a)$ follows from Lemma~\ref{lem: bounds.contractive.prod}, $(b)$ follows since $\sum_{n=0}^{N-1}\alpha_n \geq [(N+1)^{1-\alpha} - 1]/(1-\alpha).$

Therefore, for all $N\geq0,$
\begin{align*}
    \bE\|\Delta_N\|_\infty \leq C^{(1)}_Q e^{-\frac{1-\gamma}{1-\alpha}N^{1-\alpha}} + C^{(2)}_Q\sqrt{\beta_n},
\end{align*}
where  
\[
    C^{(1)}_Q := 2\Cm e^{\frac{(1-\gamma)}{(1-\alpha)}} \quad \text{and} \quad C^{(2)}_Q := \frac{\gamma}{\theta} e^{\frac{\theta\Cm}{\gamma}}C_g\CGalpha^{0,1/2}.
\]
This completes the proof of Theorem~\ref{thm:2TS.Rates}.
\end{proof}

\begin{proof}[Proof of Lemma~\ref{lem: bounds.contractive.prod}] We prove the three statements one by one.
    \paragraph{Statement 1:} Recall that for $0\leq k<N,$
    \[
        \Gbeta_{k:N} := \prod_{n=k}^{N-1}\left( 1-2\lambda\beta_n + 3\beta^2_n \right), \quad \text{and} \quad \Galpha_{k:N} := \prod_{n=k}^{N-1}\left( 1-(1-\gamma)\alpha_n \right)
    \]
    Using the fact that $1-x\leq e^{-x},$ we have
    \begin{align}
        \Gbeta_{k:N} \leq \prod_{n=k}^{N-1}e^{-2\lambda\beta_n + 2\beta^2_n } = \left(e^{-2\lambda\sum_{n=k}^{N-1}\beta_n} \right)\cdot\left(e^{3\sum_{n=k}^{N-1}\beta^2_n} \right)
    \end{align}
    and
    \begin{align}
        \Galpha_{k:N} := \prod_{n=k}^{N-1}\exp\left(-(1-\gamma)\alpha_n \right) = \exp\left(-(1-\gamma)\sum_{n=k}^{N-1}\alpha_n \right).
    \end{align}
    Taking $\beta_n = (n+1)^{-\beta},$ with $\beta\in(1/2,1)$ gives
    \begin{multline}
        \sum_{n=k}^{N-1}\beta^2_n  \leq \sum_{n=0}^{\infty}\frac{1}{(n+1)^{2\beta}} = 1 + \sum_{n=1}^{\infty}\frac{1}{(n+1)^{2\beta}} 
        \\
        \leq 1 + \sum_{n=1}^{\infty}\int_n^{n+1}\frac{1}{x^{2\beta}}dx = 1 + \int_1^\infty \frac{1}{x^{2\beta}}dx = 1 + \frac{1}{2\beta-1} = \frac{2\beta}{2\beta-1},
    \end{multline}
    Hence, setting $\CGbeta:=e^{\left(\frac{6\beta}{2\beta-1}\right)}$ gives 
    \begin{align}\label{e: iterated.beta.prod.exp}
        \Gbeta_{k:N} \leq \CGbeta e^{-2\sum_{n=k}^{N-1}\beta_n} \quad \text{and} \quad \Galpha_{k:N} \leq e^{-(1-\gamma)\sum_{n=k}^{N-1}\alpha_n}.
    \end{align}
    
    This completes the proof of \textbf{Statement 1}. Now, we focus on the next set of bounds. 
    
    \paragraph{Statement 2:} Note that, for $k\leq n<N,$
    \begin{align}\label{e: beta.sum.split}
        \sum_{n=k}^{N-1}\beta_nf_n\cdot e^{-2\lambda\sum_{m=n+1}^{N-1}\beta_m } \leq \bigg[ \max_{k\leq n < N} f_n\ e^{-\lambda\sum_{m=n+1}^{N-1}\beta_m}\bigg] \bigg( \sum_{n=k}^{N-1}\beta_n\ e^{-\lambda\sum_{m=n+1}^{N-1}\beta_m} \bigg).`
    \end{align}
    We bound these factors one by one. First, define $t_{-1}:=0$ and
$t_n:=\sum_{m=0}^{n}\beta_m$. Then, for $0\leq k<N$,
\begin{align}\label{e: beta.sum.B}
    \sum_{n=k}^{N-1}\beta_n
    e^{-\lambda\sum_{m=n+1}^{N-1}\beta_m}
    \leq
    \sum_{n=0}^{N-1}
    e^{-\lambda(t_{N-1}-t_n)}(t_n-t_{n-1}).
\end{align}
Since $0\leq\beta_n\leq\beta_0$, for every
$t\in[t_{n-1},t_n]$ we have $t_n\leq t+\beta_0\leq 1$. Thus,
for $\lambda>0$,
\begin{align}\label{e: beta.sum.Riemann}
    \sum_{n=0}^{N-1}
    e^{-\lambda(t_{N-1}-t_n)}(t_n-t_{n-1})
    &\leq e^{\lambda\beta_0}
    \sum_{n=0}^{N-1}
    \int_{t_{n-1}}^{t_n}e^{\lambda(t-t_{N-1})}\,dt 
    \nonumber\\
    &= e^{\lambda\beta_0}
    \int_0^{t_{N-1}}e^{\lambda(t-t_{N-1})}\,dt 
    \nonumber\\
    &= \frac{e^{\lambda\beta_0}}{\lambda}
    \left(1-e^{-\lambda t_{N-1}}\right)
    \leq \frac{e^{\lambda}}{\lambda}.
\end{align}
    
    Recall that $f_n = \alpha_n^a\ \beta_n^b = (n+1)^{-(a\alpha+b\beta)}$ and define for each $k\leq n < N:$
    \begin{equation}
        h_n := f_n\ e^{-\lambda\sum_{m=n+1}^{N-1}\beta_m}. 
    \end{equation}
    
    Now, consider the following two cases:

    \paragraph{Case 1: When $n\geq \ceil{N/2}.$} 

    We have $n+1\geq N/2 + 1>(N+1)/2.$ Thus,
    \begin{equation}
        \frac{h_n}{f_N} =\left(\frac{f_n}{f_N}\right)\cdot e^{-\lambda\sum_{m=n+1}^{N-1}\beta_m} \overset{(a)}{\leq} \frac{f_n}{f_N} = \left(\frac{N+1}{n+1}\right)^{a\alpha+b\beta} < 2^{-(a\alpha+b\beta)},
    \end{equation}
    where $(a)$ follows since $e^{-\lambda\sum_{m=n+1}^{N-1}\beta_m} \leq 1.$

    \paragraph{Case 2: When $n\leq  \floor{N/2}.$} Here, we have
    \begin{align}\label{e: case.B}
        \sum_{m=n+1}^{N-1} \beta_m \geq \sum_{m=\ceil{N/2}}^{N-1}\beta_m = \sum_{m=\ceil{N/2}}^{N-1}\frac{1}{(m+1)^\beta} \geq \sum_{m=\ceil{N/2}}^{N-1}\frac{1}{N^\beta} >\frac{N}{2N^\beta}.
    \end{align}
    Using the fact that $f_n\leq 1,$ we have
    \begin{multline}
        \frac{h_n}{f_N} = \frac{f_n}{f_N}\cdot e^{-\lambda\sum_{m=n+1}^{N-1}\beta_m} \overset{(a)}{\leq} \frac{1}{f_N}\cdot e^{-\lambda\sum_{m=n+1}^{N-1}\beta_m} 
        =(N+1)^{a\alpha+b\beta}\cdot e^{-\lambda\sum_{m=n+1}^{N-1}\beta_m} 
        \\
        \overset{(b)}{\leq} (N+1)^{a\alpha+b\beta}\cdot e^{-\frac{\lambda}{2}N^{1-\beta}} \overset{(c)}{\leq} {(2N)}^{a\alpha+b\beta}\cdot e^{-\frac{\lambda}{2}N^{1-\beta}},
    \end{multline}
    where $(a)$ follows since $f_n\leq 1,$  $(b)$ follows from \eqref{e: case.B}, and $(c)$ follows since $N\geq 1.$ Finally, we can show using calculus that for $p>0,$ the function $x\mapsto (2x)^p\cdot \exp(-\frac{\lambda}{2}x^{1-\beta})$ attains maximum at $x=\left(\frac{2p}{\lambda(1-\beta)}\right)^{1/(1-\beta)}$ to get
    \begin{align}
        \frac{h_n}{f_N} & \leq \sup_{x\in \bR}\left[ (2x)^{a\alpha+b\beta}\cdot e^{-\frac{\lambda}{2}x^{1-\beta}}\right] 
        \nonumber\\
        & = (2x)^{a\alpha+b\beta}\cdot e^{-\frac{\lambda}{2}x^{1-\beta}}\bigg|_{x=\left(\frac{2p}{\lambda(1-\beta)}\right)^{\frac{1}{1-\beta}}} = \left[ \frac{4(a\alpha+b\beta)}{2^\beta \lambda e(1-\beta)}\right]^{\frac{a\alpha+b\beta}{1-\beta}}.
    \end{align}
    
    Hence, for all $k\leq n<N,$ we have 
    \begin{equation}\label{e: beta.prod.split.max}
        h_n \leq C_{a,b} f_N, \quad \text{where } \quad C_{a,b}:= \max\left\{ 2^{-(a\alpha+b\beta)}, \left[ \frac{4(a\alpha+b\beta)}{2^\beta\lambda e(1-\beta)}\right]^{\frac{a\alpha+b\beta}{1-\beta}}\right\}.
    \end{equation}
    Defining $\CGbeta^{a,b}:= C_{a,b}\CGbeta (e^{\lambda}/\lambda),$ we conclude that
    \begin{align}
        \sum_{n=k}^{N-1}\beta_nf_n\ \Gbeta_{n+1:N} & \overset{(a)}{\leq} \CGbeta\bigg[ \max_{k\leq n < N} f_n\ e^{-\lambda\sum_{m=n+1}^{N-1}\beta_m} \bigg] \bigg( \sum_{n=k}^{N-1}\beta_n\ e^{-\lambda\sum_{m=n+1}^{N-1}\beta_m} \bigg)
        \nonumber\\
        & \overset{(b)}{\leq} \CGbeta\bigg[ \max_{k\leq n < N} f_n\ e^{-\lambda\sum_{m=n+1}^{N-1}\beta_m} \bigg]\left(\frac{e^{\lambda}}{\lambda}\right)
        \nonumber\\
        & \overset{(c)}{\leq} \CGbeta\cdot \left[\max_{k\leq n<N} C_{a,b}f_N\right]\cdot\left(\frac{e^{\lambda}}{\lambda}\right) = \CGbeta^{a,b} f_N,
    \end{align}
    where $(a)$ is obtained by combining from~\eqref{e: iterated.beta.prod.exp} and~\eqref{e: beta.sum.split}, $(b)$ follows by combining~\eqref{e: beta.sum.B} and~\eqref{e: beta.sum.Riemann}, and $(c)$ follows from~\eqref{e: beta.prod.split.max}.

    To finish the proof of \textbf{Statement 2}, we write
     \begin{multline}\label{e: alpha.sum.split}
        \sum_{n=k}^{N-1}\alpha_nf_n\cdot e^{-(1-\gamma)\sum_{m=n+1}^{N-1}\alpha_m } 
        \leq \bigg[\max_{k\leq n < N} f_n\ e^{-\frac{(1-\gamma)}{2}\sum_{m=n+1}^{N-1}\alpha_m} \bigg] 
        \\
        \times\bigg( \sum_{n=k}^{N-1}\alpha_n\ e^{-\frac{(1-\gamma)}{2}\sum_{m=n+1}^{N-1}\alpha_m}) \bigg).
    \end{multline}
    Treating the second factor in the above expression gives
    \begin{align}\label{e: alpha.sum.Riemann}
        \sum_{n=k}^{N-1}\alpha_n\,
        e^{-\frac{1-\gamma}{2}\sum_{m=n+1}^{N-1}\alpha_m}
        &= \sum_{n=k}^{N-1}
        e^{-\frac{1-\gamma}{2}(t_{N-1}-t_n)}
        (t_n-t_{n-1})
        \nonumber\\
        &\leq e^{\frac{1-\gamma}{2}\alpha_0}
        \int_{t_{k-1}}^{t_{N-1}}
        e^{\frac{1-\gamma}{2}(t-t_{N-1})}\,dt
        \nonumber\\
        &= \frac{2e^{\frac{1-\gamma}{2}}}{1-\gamma}
        \left(1-e^{-\frac{1-\gamma}{2}(t_{N-1}-t_{k-1})}\right)
        \leq \frac{2e^{\frac{1-\gamma}{2}}}{1-\gamma}.
    \end{align}
    where we let $t_n:=\sum_{m=0}^n\alpha_m$ and $t_{-1}=0.$ 
    
    Now, for the second factor, we define $\bar{h}_n := f_n\ e^{-\frac{(1-\gamma)}{2}\sum_{m=n+1}^{N-1}\alpha_m},$ and claim that there exists a constant $\CGalpha^{a,b}>0,$ such that for all $k\leq n<N,$ $\bar{h}_n \leq \CGalpha^{a,b}\ f_N.$  

    Using a similar analysis as earlier, we can show that 
    \begin{equation}
        \frac{\bar{h}_n}{f_N} \leq \begin{cases} \sup_{x\in\bR} \left[ (2x)^{a\alpha+b\beta}\cdot e^{\frac{(1-\gamma)}{2}x^{1-\alpha}} \right] &, \text{when }n\leq \floor{N/2}
        \\
        2^{-(a\alpha+b\beta)}&, \text{when }n\geq \ceil{N/2}
        \end{cases}.
    \end{equation}
    Since the function $x\mapsto (2x)^{a\alpha+b\beta}\cdot e^{\frac{(1-\gamma)}{2}x^{1-\alpha}}$ attains global maximum at $x= \left[ \frac{2(a\alpha+b\beta)}{(1-\gamma)(1-\alpha)}\right]^{\frac{1}{1-\alpha}},$ we conclude that
    \begin{align}\label{e: alpha.prod.split.max}
        \max_{k\leq n\leq N} \bar{h}_n \leq f_N\cdot\max\left\{2^{-(a\alpha+b\beta)}, \left[ \frac{4(a\alpha+b\beta)}{2^\alpha e(1-\gamma)(1-\alpha)}\right]^{\left(\frac{a\alpha+b\beta}{1-\alpha}\right)} \right\}.
    \end{align}
    Thus, combining~\eqref{e: iterated.beta.prod.exp}, \eqref{e: alpha.sum.split}, \eqref{e: alpha.sum.Riemann}, and~\eqref{e: alpha.prod.split.max} gives
    \begin{align}
        \sum_{n=k}^{N-1} \alpha_nf_n\ \Galpha_{n+1:N} \leq \CGalpha^{a,b}\ f_N,
    \end{align}
    where
    \[
        \CGalpha^{a,b}:= \frac{2e^{\frac{1-\gamma}{2}}}{1-\gamma}\cdot\max\left\{2^{-(a\alpha+b\beta)}, \left[ \frac{4(a\alpha+b\beta)}{2^\alpha e(1-\gamma)(1-\alpha)}\right]^{\left(\frac{a\alpha+b\beta}{1-\alpha}\right)} \right\}.
    \]
    This completes the proof.
\end{proof}

\section{Analysis for the One-Timescale Fixed-Point Method}
We discuss the stability, asymptotic  convergence, and the convergence rates of the one-timescale fixed-point method in \eqref{e:1TS.algorithm.update}.

\subsection{Stability}
We first show that the sequence \( (x_n) \) remains within the compact set 
\begin{equation}
\label{e:K.defn}
    \cK := [C_\ell, C_u]^{SA},
\end{equation} 
where $C_\ell$ and $C_u$ are as in \eqref{e:Cl.Cu.defn}. Since \( C_\ell > 0 \), $\cK$ lies in the interior of the positive orthant \( \bRp^{SA} \) .

\begin{lemma}
\label{lem:xn.Bd}
    Let $x_0 \in \cK.$ Then,  $x_n \in \cK$ and $\hF_n \in [C_\ell, C_u]$ for all $n \geq 0.$
\end{lemma}
\begin{proof}
    We use induction to first prove that $x_n(s, a) \geq C_\ell$ for all $s, a$ and all $n \geq 0.$ It trivially follows that $x_0(s, a) \geq C_\ell$ for all $s, a.$ Now, for some $n \geq 0,$ suppose that $x_n(s, a) \geq C_\ell$ for all $(s, a) \in \cS \times \cA.$ Then, for any $s, a,$ and $\sp,$
    \begin{align}
        \exp\left(-\frac{\theta}{\gamma} r(s, a) + \gamma \min_{\ap} \ln x_n(\sp, \ap) \right) 
        \overset{\tn{a}}{\geq} {} & \exp\left(-\frac{\theta}{\gamma} \|r\|_\infty + \gamma \ln C_\ell\right) \nonumber \\
        \overset{\tn{b}}{\geq} {} & C_\ell, \label{e:F.Bd}
    \end{align}
    where (a) follows since $r(s, a) \leq \|r\|_\infty$ and $x_n(s, a) \geq C_\ell$ for all $s, a,$ while (b) holds since the definition of $C_\ell$ implies $(1 - \gamma)\ln C_\ell = -\frac{\theta}{\gamma} \|r\|_\infty.$ Using this last inequality, it follows from the update rule in \eqref{e:1TS.algorithm.update} that $x_{n + 1}(s, a) \geq C_\ell$ for all $(s, a) \in \cS \times \cA,$ completing the induction proof. The same argument also show that $\hF_n \geq C_\ell$ for all $n \geq 0.$

    A similar induction argument shows that $x_n(s, a) \leq C_u$ for all $s, a$ and $\hF_n \leq C_u$ for all $n \ge 0.$  The overall claim now follows. 
\end{proof}

Building on the above result, it is easy to see that $F$ maps $\cK$ to itself.

\begin{corollary}
    \label{cor:x.in.K.implies.F(x).in.K}
    Let $K$ be defined as in \eqref{e:K.defn}. Then, $x \in \cK \implies F(x) \in \cK.$
\end{corollary}
\begin{proof}
    The definition of $F$ given in \eqref{e:F.defn} shows that
    \[
        F(x)(s, a) = \bE_{\sp \in \cP(\cdot|s, a)} \bigg[ \exp\left(-\frac{\theta}{\gamma} r(s, a) + \gamma \min_{\ap} \ln x_n(\sp, \ap) \right) \bigg].
    \]
    Hence, by arguing as in the proof of Lemma~\ref{lem:xn.Bd}, the desired result is easy to see. 
\end{proof}

\subsection{Asymptotic Convergence}

We now discuss \eqref{e:1TS.algorithm.update}'s asymptotic behavior. Let $(\cF_n)_{n \geq 0}$ be the filtration of $\sigma$-fields, where
\[
    \cF_n := \sigma(x_0, s_0, a_0, \ldots, s_n, a_n).
\]
The update in \eqref{e:1TS.algorithm.update} can then be rewritten as 
\begin{equation}
\label{e:1TS.algorithm.update.equiv}
    x_{n + 1} = x_n + \alpha_n \left[ e_{s_n, a_n}\bigg((F(x_n))_{s_n, a_n} - x_n(s_n, a_n) \bigg) + M_{n + 1}\right],
\end{equation}
where
\begin{equation}
\label{e:M.n.1.TS.defn}
    M_{n + 1} = e_{s_n, a_n}\bigg(\hF_n - (F(x_n))_{s_n, a_n} \bigg).
\end{equation}
 Then, it is straightforward to see that the associated limiting ODE is given by 
\begin{equation}
\label{e:x.lim.ODE}
        \dot{x}(t) = D_{\pi_{x(t)}} [F(x(t)) - x(t)],
\end{equation}
where $D_{\pi_x}$ is the diagonal matrix made up the stationary distribution $\pi_x$ (defined below \eqref{e:Cl.Cu.defn}). Let
\begin{equation}
\label{e:ell.cK.u.cK.Defn}
    \ell_\cK := \min_{x \in \cK} \min_{i \in \cS \times \cA} D_{\pi_x}(i, i) \quad \text{ and } \quad u_\cK := \max_{x \in \cK} \max_{i \in \cS \times \cA} D_{\pi_x}(i, i).
\end{equation}
Since $\cK$ is compact and bounded away from $0$, it follows that
\begin{equation}
    0 < \ell_\cK \leq u_\cK < \infty.
\end{equation}

For \eqref{e:x.lim.ODE} to be well-posed—i.e., to ensure existence and uniqueness of solutions on $[0,\infty)$, along with continuous dependence on initial conditions—standard results typically require either global Lipschitz continuity of the driving function or, more generally, local Lipschitz continuity together with the property that solution trajectories remain within a compact set. 

In our setting, $F$ involves terms of the form $[\min_{s,a} x(\sp,\ap)]^\gamma$ with $\gamma \in [0,1)$, and hence is not globally Lipschitz; moreover, it fails to be locally Lipschitz near the boundary of $\bRp^{SA}$. We therefore adopt the second approach to show \eqref{e:x.lim.ODE}'s well-posedness: we show that $F$ is locally Lipschitz on a neighbourhood of $\cK$ in $\bRp^{SA}$ and that $\cK$ is positively invariant under the induced dynamics.

\begin{lemma}
\label{lem:F.Lipschitz}
    The function $F$ is Lipschitz continuous on an open neighbourhood of $\cK$ in $\bRp^{SA}.$
\end{lemma}
\begin{proof}
Let $\delta := C_\ell/2 > 0$; the inequality holds since $C_\ell > 0$. Then,
\[
    \cK \subseteq \cK^{\delta} := (C_\ell - \delta,\, C_u + \delta)^{SA} \subset \bRp^{SA}.
\]
It suffices to show that $F$ is Lipschitz continuous on $\cK^{\delta}$.

Let $h: \bR \to \bR$ be given by $h(z) = z^\gamma.$ Then, for $z_1, z_2 \geq C_\ell/2 > 0,$ there exists some $c$ between $z_1$ and $z_2$ such that 
\begin{align}
    |h(z_1) - h(z_2)| \overset{\tn{a}}{=} {} & |h'(c)| |z_1 - z_2|, \nonumber \\
    \overset{\tn{b}}{=} {} & \frac{\gamma}{c^{1 - \gamma}} |z_1 - z_2| \nonumber\\
    \leq {} & \frac{2^{1 - \gamma} \gamma}{C_\ell^{1 - \gamma}} |z_1 - z_2|, \label{e:h.mean.value.theorem}
\end{align}
where (a) follows from the mean value theorem, while (b) follows since $c \geq \min\{z_1, z_2\} \geq C_\ell/2.$

Therefore, for any $(s, a) \in \cS \times \cA,$
\begin{align}
    |(F(x_1)(s, a) & - F(x_2)(s, a)| \nonumber \\
    \overset{\tn{a}}{\leq} {} & \sum_{\sp \in \cS} \cP (\sp|s, a) \exp\left(-\frac{\theta}{\gamma} r(s, a)\right) \bigg| \Big[\min_{a'} x_1(\sp, \ap)\Big]^\gamma - \Big[\min_{a'} x_2(\sp, \ap)\Big]^\gamma \bigg| \\
    \overset{\tn{b}}{\leq} {} & \frac{ 2^{1- \gamma}  \gamma}{C_\ell^{1- \gamma}} \sum_{\sp \in \cS} \cP(\sp|s, a) \exp\left(-\frac{\theta}{\gamma} r(s, a)\right) \bigg|\min_{a'} x_1(\sp, \ap) - \min_{a'} x_2(\sp, \ap) \bigg|\\
    \overset{\tn{c}}{\leq} {} & \frac{2^{1- \gamma} \gamma}{C_\ell^{1- \gamma}} \exp\left(\frac{\theta}{\gamma} \|r\|_\infty\right)  \|x_1 - x_2\|_\infty, 
\end{align}
where (a) follows from \eqref{e:F.defn}, (b) follows from \eqref{e:h.mean.value.theorem}, while (c) holds from since $r(s, a) \geq -\|r\|_\infty$ and since a standard argument shows
\[
    \bigg|\min_{a'} x_1(\sp, \ap) - \min_{a'} x_2(\sp, \ap) \bigg| \leq \|x_1 - x_2\|_\infty
\]
Taking the supremum over $(s,a)$ yields
\begin{equation}
    \|F(x_1) - F(x_2)\|_\infty \leq \frac{\gamma}{C_\ell^{1- \gamma}} \exp\left(\frac{\theta}{\gamma} \|r\|_\infty \right) \|x_1 - x_2\|_\infty,
\end{equation}
which proves that $F$ is Lipschitz continuous on $\cK^\delta,$ as desired. 
\end{proof}

Our next result establishes that \(\cK\) is positively invariant for the limiting ODE in \eqref{e:x.lim.ODE} and that \(\xstr\) is globally asymptotically stable relative to \(\cK\).

\begin{theorem}
\label{thm:x.ODE.prop} 
    The set $\cK$ defined in \eqref{e:K.defn} contains the fixed point $\xstr$ of $F$ and is positively invariant under the ODE in \eqref{e:x.lim.ODE}. Moreover, $\xstr$ is a globally asymptotically stable equilibrium of \eqref{e:x.lim.ODE} relative to $\cK$; that is, for every initial condition $x(0) \in \cK$, the corresponding trajectory satisfies
    \[
        x(t) \to \xstr \quad \text{as } \quad t \to \infty.
    \]
\end{theorem}
\begin{proof}
    We first show that $\xstr \in \cK$. Clearly, $\cK$ is a non-empty convex compact subset of $\bRp^{SA}.$ Moreover, $F$ is continuous and Lemma~\ref{cor:x.in.K.implies.F(x).in.K} implies that $F(\cK) \subseteq \cK$. It thus follows from Brouwer's fixed-point theorem that $F$ has a fixed point in $\cK$. Since $\xstr$ is the unique fixed point of $F$ on $\bRp^{SA}$, it follows that $\xstr \in \cK$, as desired. 

    Next, we show that $\cK$ is positively invariant under the ODE in \eqref{e:x.lim.ODE}. Under Assumption~\ref{ass:1TS.markov}, the stationary distribution $\pi_x$ depends Lipschitz continuously on $x\in\cK$, since the state-action space is finite and the kernels $K_x$ are irreducible. Together with Lemma~\ref{lem:F.Lipschitz}, this implies that the vector field
    \[
        f(x):=D_{\pi_x}[F(x)-x]
    \]
    is Lipschitz continuous on $\cK$ and admits a locally Lipschitz extension to a neighborhood of $\cK$. Hence, the associated ODE has unique solutions.
    
    Now, let $x\in\cK$. If $x_i=C_\ell$, then
    \[
        f_i(x)=D_{\pi_x}(i,i)[F_i(x)-C_\ell]\geq0,
    \]
    since $F(x)\in\cK$. Similarly, if $x_i=C_u$, then $f_i(x)\leq0$. These inequalities show that $f(x)$ belongs to the tangent cone of $\cK$ at every $x\in\cK$. Therefore, by Nagumo's invariance theorem~\cite[Theorem~11.2.3]{aubin2011viability}, together with uniqueness of solutions, $\cK$ is positively invariant under the ODE in \eqref{e:x.lim.ODE}. Since $\cK$ is compact, it then follows that all solutions starting in $\cK$ exist for all $t \geq 0.$
    
    By Proposition~\ref{prop:TFcontract1}, the operator $F$ is monotone and homogeneous of degree $\gamma$, i.e.,   
    for all $x,y \in\bRp^{SA}$ and $c>0$,
    \[
        x \leq y \implies F(x)\leq F(y),
    \qquad F(cx)=c^\gamma F(x).
    \]
    We use these properties to establish the global asymptotic stability of $\xstr$.

    Let $x$ be any solution trajectory of \eqref{e:x.lim.ODE}, starting in $\cK.$  Since $\xstr \in \cK,$ the quantities 
    \[
        m(t) := \min_{i \in \cS \times \cA} \frac{x_i(t)}{\xstr_i} \text{ and } M(t) := \max_{i \in \cS \times \cA} \frac{x_i(t)}{\xstr_i}
    \]
    are well defined for any $t \geq 0$. These definitions imply
    \begin{equation}
    \label{e:m(t).x(t).M(t).rel}
        m(t) \xstr \leq x(t) \leq M(t) \xstr.
    \end{equation}
    To show $\lim_{t \to \infty} x(t) =  \xstr,$ it then suffices to show that 
    \begin{equation}
    \label{e:m.M.lim.1}
        1 \leq \liminf_{t \to \infty} m(t) \leq \limsup_{t \to \infty} M(t) \leq 1.
    \end{equation}

    We first show that $\liminf_{t\to\infty}m(t)\geq 1$. By the homogeneity and monotonicity of $F$,
    \begin{equation}
    \label{e:m.F.rel}
    m(t)^\gamma F(\xstr)=F(m(t)\xstr)\leq F(x(t)).
    \end{equation}
    Let $D_+m(t)$ denote the lower right Dini derivative of $m(t)$, defined by
    \[
        D_+m(t):=\liminf_{h\downarrow0}\frac{m(t+h)-m(t)}{h},
    \]
    and let
    \[
    I(t):= \underset{i \in \cS \times \cA}{\arg\min}\frac{x_i(t)}{\xstr_i}.
    \]
    Then, by Danskin's theorem \citep[Theorem~1]{danskin1966theory} (cf. \citep[Lemma~2.2.6]{davydov2025contraction}) and the ODE dynamics,
    \begin{align*}
    D_+m(t)
    &=\min_{i\in I(t)}\frac{\dot x_i(t)}{\xstr_i}\\
    &=\min_{i\in I(t)}
    D_{\pi_{x(t)}}(i,i)
    \frac{F_i(x(t))-x_i(t)}{\xstr_i}\\
    &\geq\min_{i\in I(t)}
    D_{\pi_{x(t)}}(i,i)
    \frac{m(t)^\gamma F_i(\xstr)-x_i(t)}{\xstr_i}\\
    &=\min_{i\in I(t)}
    D_{\pi_{x(t)}}(i,i)
    [m(t)^\gamma-m(t)],
    \end{align*}
    where the inequality follows from \eqref{e:m.F.rel}, and the last equality uses $F(\xstr)=\xstr$ and $x_i(t)=m(t)\xstr_i$ for every $i\in I(t)$.
    
    Since $0<\ell_\cK\leq D_{\pi_{x(t)}}(i,i)\leq u_\cK$, we distinguish two cases. If $m(t)\leq1$, then $m(t)^\gamma-m(t)\geq0$, and hence
    \[
    D_+m(t)\geq\ell_\cK[m(t)^\gamma-m(t)].
    \]
    If $m(t)>1$, then $m(t)^\gamma-m(t)<0$, so
    \[
    D_+m(t)\geq u_\cK[m(t)^\gamma-m(t)].
    \]
    Combining the two cases gives
    \begin{equation}
    \label{e:m.Dini.bound}
    D_+m(t)\geq
    \begin{cases}
    \ell_\cK[m(t)^\gamma-m(t)], & m(t)\leq1,\\
    u_\cK[m(t)^\gamma-m(t)], & m(t)>1.
    \end{cases}
    \end{equation}

    Next, let $z:[0,\infty)\to\bRp$ solve
    \[
        \dot z(t)=
        \begin{cases}
            \ell_\cK[z(t)^\gamma-z(t)], & z(t)\leq 1,\\
            u_\cK[z(t)^\gamma-z(t)], & z(t)>1,
        \end{cases}
    \]
    with $z(0)=m(0)>0$. Since $\gamma\in(0,1)$ and
    $0<\ell_\cK\leq u_\cK$, the right-hand side is
    locally Lipschitz on $\bRp$, positive for $z<1$,
    and negative for $z>1$. Consequently, $1$ is
    the unique globally asymptotically stable
    equilibrium on $\bRp$. Applying the Dini
    comparison lemma\footnote{The result in
    \citet[Lemma~3.4]{khalil2002nonlinear} is stated
    for the upper right Dini derivative. The
    corresponding result for the lower right Dini
    derivative follows by negating the functions.}
    to \eqref{e:m.Dini.bound} yields $m(t)\geq z(t)$
    for all $t\geq0$. Since $z(t)\to1$, we conclude
    that
    \[
        \liminf_{t\to\infty}m(t)\geq1.
    \]
    A symmetrical argument using the upper right Dini derivative of $M(t)$ shows $\limsup M(t) \leq 1.$ Separately, $m(t) \leq M(t)$ implies $\liminf_{t \to \infty} m(t) \leq \limsup_{t \to \infty} M(t).$ This completes the proof of \eqref{e:m.M.lim.1} and, hence, of the claim that  $\lim_{t \to \infty} x(t) = \xstr.$

    It remains to establish Lyapunov stability of $\xstr$. The preceding scalar comparison arguments also imply that
    \[
        m(t)\geq\min\{m(0),1\},
        \qquad
        M(t)\leq\max\{M(0),1\},
        \quad t\geq0.
    \]
    Consequently, letting
    \[
        \delta_0 :=
        \max_i\left|\frac{x_i(0)}{\xstr_i}-1\right|,
    \]
    we obtain
    \[
        (1-\delta_0)\xstr
        \leq x(t)\leq
        (1+\delta_0)\xstr,
        \qquad t\geq0.
    \]
    It follows that
    \[
        \|x(t)-\xstr\|_\infty
        \leq
        \frac{\|\xstr\|_\infty}{\min_i\xstr_i}
        \|x(0)-\xstr\|_\infty,
        \qquad t\geq0.
    \]
    Since $\xstr$ has strictly positive coordinates,
    this establishes Lyapunov stability. Combined with
    the global attractivity established above, it follows
    that $\xstr$ is globally asymptotically stable relative
    to $\cK$.
    
    Since the solution trajectory $x$ was arbitrary, it follows that $\xstr$ is a globally asymptotically stable equilibrium of \eqref{e:x.lim.ODE} relative to $\cK,$ as desired.
\end{proof}

Since $F$ is only locally Lipschitz (Lemma~6), we now show how the ODE method from \citep{borkar2008stochastic} can be adapted to prove Theorem~\ref{thm:oneTSconv}. 

\begin{proof}[\textbf{Proof of Theorem~\ref{thm:oneTSconv}}]
We verify the standard conditions required for the ODE method. 

Lemma~\ref{lem:F.Lipschitz} shows that $F$ is Lipschitz continuous on an open neighbourhood of $\cK$. In particular, this ensures that the limiting ODE \eqref{e:x.lim.ODE} is well-posed. Further, by Lemma~\ref{lem:xn.Bd} and Theorem~\ref{thm:x.ODE.prop}, both the iterate sequence $(x_n)$ and the solution trajectories of \eqref{e:x.lim.ODE} remain confined to $\cK$. The stepsize sequence $(\alpha_n)$ satisfies the Robbins--Monro conditions. 

It remains to verify the properties of the noise sequence $(M_n)$ defined in \eqref{e:M.n.1.TS.defn}. By construction,
\[
    \bE[M_{n+1} \mid \cF_n] \overset{a.s.}{=} 0, \quad n \geq 0,
\]
so $(M_n)$ is a martingale difference sequence with respect to the filtration $(\cF_n)$. Moreover, since $\hat{F}_n, F(x_n) \in \cK$ (see Lemma~\ref{lem:xn.Bd} and Corollary~\ref{cor:x.in.K.implies.F(x).in.K}), we have
\[
    \|M_{n+1}\|_2 \leq C_u - C_\ell,
\]
where $C_\ell$ and $C_u$ are defined in \eqref{e:Cl.Cu.defn}. Consequently,
\[
    \bE\!\left[\|M_{n+1}\|_2^2 \mid \cF_n\right] \leq (C_u - C_\ell)^2.
\]

Thus, all the conditions of the ODE method are satisfied. It follows that $(x_n)$ tracks the solution trajectories of the limiting ODE \eqref{e:x.lim.ODE}. Since $\xstr$ is the unique globally asymptotically stable equilibrium of \eqref{e:x.lim.ODE} relative to $\cK$, we conclude that $x_n \to \xstr$ almost surely.
\end{proof}

\subsection{Convergence Rates: Scalar Case}
Here, we analyze the convergence rate of a scalar stochastic
approximation model associated with the one-timescale update in
\eqref{e:1TS.algorithm.update}. The aim is to isolate the challenges posed by
the power-law structure. 

We begin by considering the one-state, one-action reduction of the function $F$ defined in \eqref{e:F.defn}. Let the corresponding reward be denoted by $r \in \bR$. Then, $F:\bRp \to \bRp$ takes the form
\[
    F(x) = \exp\!\left(-\frac{\theta}{\gamma} r \right) x^\gamma.
\]
Solving $F(x^*) = x^*$ yields
\[
    x^* = \exp\!\left(-\frac{\theta}{\gamma(1 - \gamma)} r \right),
\]
so that
\begin{equation}
    \label{e:F.power.law.form}    
    F(x) = \left(\frac{x}{x^*}\right)^\gamma x^*.
\end{equation}

In the literal single-state, single-action model with deterministic rewards, the update in \eqref{e:1TS.algorithm.update.equiv} becomes deterministic. To study the effect of stochastic approximation noise while preserving the same power-law drift, we consider the following  scalar recursion related to \ref{e:1TS.algorithm.update}:
\begin{equation}
    \label{e:1TS.algorithm.update.scalar}
    x_{n+1}
    =
    x_n+\alpha_n\left[F(x_n) - x_n + \zeta_{n+1}\right],
\end{equation}
where $(\zeta_{n})_{n \geq 1}$ is a martingale-difference sequence with respect to the
filtration $(\mathcal{F}_n)_{n \geq 0}$, defined by
\begin{equation}
    \label{e:1TS.scalar.filtration}
    \mathcal{F}_n := \sigma(x_0, \zeta_1, \ldots, \zeta_n).
\end{equation}
For this new scalar update, we presume that there exists constants $C_\ell$ and $C_u,$ similar to \eqref{e:Cl.Cu.defn} 
such that $0 < C_\ell \leq C_u$ and $\xstr \in \cK := [C_\ell, C_u].$ As in Lemma~\ref{lem:xn.Bd}, we assume that $x_0 \in \cK$ implies $x_n \in \cK$ for all $n \geq 0.$ We also assume that the martingale differences satisfy
\begin{equation}
\label{e:zeta.n.Bd}
    |\zeta_{n}| \leq C_u - C_\ell, \qquad n \geq 0. 
\end{equation}

For the convergence and convergence-rate analysis of
\eqref{e:2TS.algorithm.update}, the global contraction of $T$ plays a
crucial role. In particular, it allows us to use
$\|Q(t)-\Qstr\|_\infty$ as a global Lyapunov function in \eqref{e:T.contraction.impact} for the limiting
ODE associated with the $Q_n$-update in \eqref{e:2TS.algorithm.update}.
At the discrete-time level, the same contraction property yields the
strict shrinkage factor in \eqref{e:T.contraction.impact.rate}, which
provides a negative drift toward $\Qstr$ and thereby enables the
convergence-rate analysis.

In contrast, the power-law form of $F$ in \eqref{e:F.power.law.form}
implies that $F'(x)\to\infty$ as $x\downarrow 0$. Thus, $F$ is
contractive only on the subset of $\bRp$ that is bounded away from the origin. Indeed, for
any $\epsilon>0$ such that $\gamma+\epsilon<1$, there exists
$\delta>0$ such that
\[
    |F(x)-F(x')| \leq (\gamma+\epsilon)|x-x'|,
    \qquad
    x,x' \geq \xstr-\delta .
\]
This lack of global contractivity is the main obstacle in obtaining
convergence rates for our one-timescale method. We now explain how this can be overcome in the scalar case.

\begin{theorem}\label{lem: scalar.1TS.rate}
    Let
    \[
        \underline y := \frac{C_\ell}{\xstr},
        \qquad
        \overline y := \frac{C_u}{\xstr}.
    \]
    Define
    \[
        C_1 :=
        \begin{cases}
        \dfrac{\underline y-\underline y^\gamma}{\underline y-1},
        & \underline y \neq 1, \\[2ex]
        1-\gamma,
        & \underline y = 1.
        \end{cases}
    \]
    Let
    \[
        \alpha_n=\frac{1}{2C_1(n+1)}.
    \]
    Then
    \[
        \bE\left|\frac{x_n}{\xstr}-1\right|
        \leq
        \sqrt{\frac{\widetilde C_2(1+\ln n)}{n}}
        =
        \tilde O(n^{-1/2}),
    \]
    where
    \[
        \widetilde C_2
        :=
        \frac{1}{4C_1^2}
        \left[
        \max\left\{
        \left(\frac{C_u}{\xstr}\right)^{2\gamma},
        \left(\frac{C_u}{\xstr}\right)^2
        \right\}
        +
        \left(\frac{C_u}{\xstr}\right)^2
        \right],
    \]
    and $\tilde O$ hides logarithmic terms.
\end{theorem}
\begin{proof}
    Let $y_n := \frac{x_n}{\xstr}.$ Since $0 < C_\ell \leq \xstr \leq C_u < \infty,$ we have
    \begin{equation}
    \label{e:y.n.Bd}
        0<\underline y
        \leq y_n
        \leq \overline y
        <\infty,
        \qquad n\geq 0.
    \end{equation}
    Dividing \eqref{e:1TS.algorithm.update.scalar} by $\xstr$ and using
    \eqref{e:F.power.law.form}, we obtain
    \begin{align}
        y_{n+1}
        &=
        y_n
        +
        \alpha_n
        \left[
        \frac{F(x_n)}{\xstr}
        -
        y_n
        +
        \frac{\zeta_{n+1}}{\xstr}
        \right] \nonumber \\
        &=
        y_n
        +
        \alpha_n
        \left[
        y_n^\gamma-y_n
        +
        \frac{\zeta_{n+1}}{\xstr}
        \right].
        \label{e:y.n.update.rule}
    \end{align}
    Let
    \[
        z_n := y_n-1.
    \]
    Then
    \[
        z_{n+1}
        =
        z_n
        +
        \alpha_n (y_n^\gamma-y_n)
        +
        \alpha_n\frac{\zeta_{n+1}}{\xstr}.
    \]
    Squaring both sides, taking conditional expectation with respect to
    $\cF_n$, and using $\bE[\zeta_{n+1}\mid \cF_n]=0$, we get
    \begin{align}
        \bE[z_{n+1}^2\mid \cF_n]
        &=
        z_n^2
        +
        2\alpha_n z_n(y_n^\gamma-y_n)
        +
        \alpha_n^2(y_n^\gamma-y_n)^2
        +
        \alpha_n^2
        \bE\left[
        \left(\frac{\zeta_{n+1}}{\xstr}\right)^2
        \middle|
        \cF_n
        \right].
        \label{e:cond.second.moment.exact}
    \end{align}
    From \eqref{e:zeta.n.Bd}, we have
    \[
        \left(\frac{\zeta_{n+1}}{\xstr}\right)^2
        \leq
        \left(\frac{C_u-C_\ell}{\xstr}\right)^2
        \leq
        \left(\frac{C_u}{\xstr}\right)^2.
    \]
    Also, using \eqref{e:y.n.Bd},
    \[
        |y_n^\gamma-y_n|
        \leq
        \max\{y_n^\gamma,y_n\}
        \leq
        \max\left\{
        \left(\frac{C_u}{\xstr}\right)^\gamma,
        \frac{C_u}{\xstr}
        \right\}.
    \]
    Therefore, with
    \[
        C_2 :=
        \max\left\{
        \left(\frac{C_u}{\xstr}\right)^{2\gamma},
        \left(\frac{C_u}{\xstr}\right)^2
        \right\}
        +
        \left(\frac{C_u}{\xstr}\right)^2,
    \]
    we obtain
    \begin{equation}
    \label{e:E.y.n+1.Bd.corrected}
        \bE[z_{n+1}^2\mid \cF_n]
        \leq
        z_n^2
        +
        2\alpha_n z_n(y_n^\gamma-y_n)
        +
        C_2\alpha_n^2.
    \end{equation}

    We now establish a uniform drift bound. Define
    \[
        h(y)
        :=
        \frac{y-y^\gamma}{y-1},
        \qquad y>0,\ y\neq 1,
    \]
    with the continuous extension $h(1):=1-\gamma$. We first show that $h$ is
    nondecreasing on $\bRp.$ For $y>0$ and $y\neq 1$,
    \[
        (y-1)^2 h'(y)
        =
        (1-\gamma)y^\gamma+\gamma y^{\gamma-1}-1.
    \]
    Multiplying by $y^{1-\gamma}>0$, define
    \[
        g(y)
        :=
        y^{1-\gamma}(y-1)^2h'(y)
        =
        (1-\gamma)y+\gamma-y^{1-\gamma}.
    \]
    Then
    \[
        g'(y)
        =
        (1-\gamma)-(1-\gamma)y^{-\gamma},
        \qquad
        g''(y)
        =
        \gamma(1-\gamma)y^{-(1+\gamma)}
        \geq 0.
    \]
    Thus $g$ is convex on $(0,\infty)$. Since
    \[
        g(1)=0,
        \qquad
        g'(1)=0,
    \]
    it follows that $g$ attains its global minimum at $1$, and hence
    $g(y)\geq 0$ for all $y>0$. Therefore, $h'(y)\geq 0$ for all
    $y>0$, $y\neq 1$, and so $h$ is nondecreasing on $(0,\infty)$.

    Since $h$ is nondecreasing and $y_n\geq \underline y$, we have
    $h(y_n)\geq C_1$. Moreover,
    \[
        y_n-y_n^\gamma
        =
        h(y_n)(y_n-1).
    \]
    Hence
    \begin{equation}
    \label{e:y_n.prod.claim.corrected}
        (y_n-1)(y_n^\gamma-y_n)
        =
        -h(y_n)(y_n-1)^2
        \leq
        -C_1(y_n-1)^2
        =
        -C_1 z_n^2.
    \end{equation}
    Substituting \eqref{e:y_n.prod.claim.corrected} into
    \eqref{e:E.y.n+1.Bd.corrected} gives
    \[
        \bE[z_{n+1}^2\mid \cF_n]
        \leq
        z_n^2
        -
        2C_1\alpha_n z_n^2
        +
        C_2\alpha_n^2.
    \]
    Taking expectation on both sides yields
    \[
        \bE z_{n+1}^2
        \leq
        (1-2C_1\alpha_n)\bE z_n^2
        +
        C_2\alpha_n^2.
    \]
    Recalling that
    \[
        \alpha_n=\frac{1}{2C_1(n+1)},
    \]
    we get
    \begin{equation}
    \label{e:z.second.moment.recursion.corrected}
        \bE z_{n+1}^2
        \leq
        \frac{n}{n+1}\bE z_n^2
        +
        \frac{\widetilde C_2}{(n+1)^2}.
    \end{equation}
    Multiplying both sides of \eqref{e:z.second.moment.recursion.corrected}
    by $n+1$ gives
    \[
        (n+1)\bE z_{n+1}^2
        \leq
        n\bE z_n^2
        +
        \frac{\widetilde C_2}{n+1}.
    \]
    Iterating the above inequality from $0$ to $n-1$, and using that the
    initial term is killed by the factor $1-2C_1\alpha_0=0$, gives
    \[
        n\bE z_n^2
        \leq
        \widetilde C_2
        \sum_{k=1}^n \frac{1}{k}.
    \]
    Therefore,
    \[
        \bE z_n^2
        \leq
        \widetilde C_2\frac{1+\ln n}{n}.
    \]
    Finally, by Jensen's inequality,
    \[
        \bE\left|\frac{x_n}{\xstr}-1\right|
        =
        \bE |z_n|
        \leq
        \sqrt{\bE z_n^2}
        \leq
        \sqrt{
        \widetilde C_2\frac{1+\ln n}{n}
        }.
    \]
    This completes the proof.
\end{proof}

\newpage

\end{document}